\def\GNQClean{1}
\documentclass{article}

\usepackage{iclr2027_conference}
\iclrfinalcopy
\usepackage{times}

\usepackage[utf8]{inputenc}
\usepackage[T1]{fontenc}
\usepackage{hyperref}
\usepackage{url}
\usepackage{booktabs}
\usepackage{amsfonts}
\usepackage{amsmath}
\usepackage{amssymb}
\usepackage{amsthm}
\usepackage{mathtools}
\usepackage{nicefrac}
\usepackage{microtype}
\usepackage{xcolor}
\usepackage{graphicx}
\usepackage{subcaption}
\usepackage{multirow}
\usepackage{pifont}
\usepackage[ruled,vlined]{algorithm2e}
\usepackage{algpseudocode}
\usepackage{wrapfig}
\usepackage{needspace,xurl}
\usepackage{caption}

\newif\ifadvisor
\ifdefined\GNQClean\advisorfalse\else\advisortrue\fi
\definecolor{revtheory}{HTML}{164A91}
\definecolor{revexperiment}{HTML}{006B67}
\definecolor{revnarrative}{HTML}{66348A}
\definecolor{revplain}{HTML}{000000}
\newcommand{\rev}[2]{{\ifadvisor\color{rev#1}\else\color{black}\fi #2}}

\hypersetup{colorlinks=true,linkcolor=black,citecolor=black,urlcolor=black,pdfauthor={},pdftitle={Auditing Information Disclosure During Large-Scale Gradient-Based Training via Gradient Uniqueness}}

\newcommand{\cmark}{\ding{51}}
\newcommand{\xmark}{\ding{55}}

\newtheorem{theorem}{Theorem}

\newtheorem{lemma}{Lemma}
\newtheorem{definition}{Definition}

\title{Auditing Information Disclosure During Large-Scale Gradient-Based Training via Gradient Uniqueness}

\author{
Sleem Abdelghafar \qquad
Maryam Aliakbarpour \qquad
Christopher Jermaine \\
\multicolumn{1}{c}{Rice University \quad
\texttt{\{msm15,maryama,cmj4\}@rice.edu}}
}

\begin{document}

\maketitle
\pagestyle{plain}
\thispagestyle{plain}

\begin{abstract}
Auditing information disclosure across every datapoint during the training of LLMs is challenging. We propose a principled, attack-agnostic approach that uses mutual information to measure what the final model reveals about a datapoint's training membership. 
We show that this final-model disclosure is upper bounded by the sum of per-iteration gradient disclosures and that, under a reasonable set of assumptions, these gradient disclosures increase with Gradient Uniqueness (GNQ), which measures how distinguishable a datapoint's gradient is relative to other gradients in the batch.
While naively computing GNQ requires forming and inverting a $P\times P$ matrix for every datapoint (for a model with $P$ parameters), we introduce Batch-Space Ghost (BS-Ghost). This efficient algorithm performs all computations in a much smaller batch space and uses ghost kernels to compute GNQ ``in-run'' for every datapoint in the training corpus, with minimal computational and memory overhead. Our experiments show the following: (i)~GNQ predicts MIA vulnerability without the need for shadow models. (ii)~Beyond membership disclosure, GNQ predicts the success of reconstruction attacks. (iii)~GNQ-guided removal and retraining identify datapoints that causally contribute to disclosure. (iv)~GNQ outperforms counterfactual memorization in text extraction and common-knowledge discrimination without the need for additional model training. (v)~For data attribution, GNQ-guided filtering reduces emergent misalignment in Qwen2.5-7B more than baselines. Further, GNQ attributes 1000 datapoints in 17 seconds---roughly $70\times$ faster than the baselines. (vi)~GNQ explains how training choices affect training-set disclosure and captures how per-datapoint disclosure emerges during training.
\end{abstract}
\vspace{-11pt}

\section{Introduction}
After training, a deployed model may reveal information about specific training datapoints.
This can happen when analysis or querying of the model can reveal training text verbatim or near-verbatim \citep{carlini2019secretsharer,carlini2021extracting,nasr2023scalable,hayes2025probabilistic},
or when personally identifiable information (PII) is embedded in model output
(e.g., leaking a phone number inside a new sentence) \citep{zhou2023entity,kim2023propile,lukas2023pii}.
Such disclosures enable training-data extraction and related privacy attacks \citep{ishihara2023survey}, and 
may violate contractual or legal constraints on data usage \citep{henderson2023fairuse,cooper2025files}.
For these reasons, it is important to be able to \emph{audit} how much disclosure an AI model poses.

We consider the problem of designing a practical approach for auditing the extent to which an AI model encodes information about each of the datapoints it was trained on---where ``information'' is defined formally, via an information-theoretic analysis.
Our goal is to design an auditing framework that meets the following requirements (Table~\ref{tab:prior_work_desiderata}):

\noindent
\emph{(C1) Auditing should be attack-agnostic.}
Measuring the extent to which a model resists a given \emph{attack} is perhaps the most common way to audit disclosure.
Examples include prompting-based extraction and string-match pipelines \citep{carlini2021extracting,nasr2023scalable,hayes2025probabilistic},
membership inference attacks (MIAs) that predict whether a candidate datapoint was in the training set \citep{shokri2017mia,mattern2023neighbourhood,xie2024recall},
and other reconstruction or PII-focused probes \citep{lukas2023pii,kim2023propile}.
However, such methods are attack-specific: failure of one attack does not imply safety against others \citep{song2021systematicprivacy,carlini2022lira}.

\noindent
\emph{(C2) Auditing should cover every training datapoint, ``in-run'' and at low cost.}
For LLM-scale datasets and models, it is impractical for auditing to require post-hoc analysis of the model with respect to every training datapoint,
or heavyweight procedures that scale with the number of prompts, generations, or training-set substrings \citep{nasr2023scalable,hayes2025probabilistic}.
Auditing should take place as the model is constructed (``in-run'') and introduce low computational and memory overhead.

\noindent
\emph{(C3) Auditing should not modify the baseline training and should audit the actual trained model.}
In realistic LLM settings, training is expensive and recipes are tightly tuned; practitioners are reluctant to alter data or optimization in ways that may affect utility,
safety behavior, or training stability.
Thus, auditing should not require changing the training dataset (e.g., inserting canaries) \citep{carlini2019secretsharer,parikh2022canaryNLU}
or running a \emph{different} training procedure to measure disclosure.

\noindent
\emph{(C4) Auditing should take into account prior/common knowledge.}
Consider ``Napoleon Bonaparte lost the Battle of Waterloo on June 18, 1815.''  Any LLM trained on a large corpus would be able to reproduce this sequence, regardless of whether it had seen the sequence during training. Declaring all such ``common knowledge'' data to be disclosed/memorized is unhelpful.

\noindent
In this paper, we propose a new method for disclosure auditing that meets the requirements above, called \emph{gradient uniqueness} (GNQ).
Our approach starts from a principled, attack-agnostic measure of disclosure: the mutual information $I(T_j;\theta_{N_r})$ between a datapoint's membership and the final model (C1).
Given a population prior over training samples unknown to the adversary, we show that this disclosure is bounded by the sum of per-iteration gradient disclosures (Equation~\eqref{eq:final-model-disclosure}). Under a reasonable set of assumptions, per-iteration disclosure increases with GNQ (Theorem~\ref{thm:gaussian-gnq}), which is computable directly from gradients within a single training run. Our Batch-Space Ghost (BS-Ghost) algorithm computes GNQ for every training datapoint ``in-run,'' with very low overhead (C2).
Moreover, GNQ does not require modifying the training algorithm, and it audits the disclosure of the model produced by the training run (C3).
GNQ accounts for common knowledge by measuring  how distinguishable a datapoint's  gradient is relative to gradients from the training population (C4).

Our specific contributions are as follows:
\begin{itemize}
\item We bound final-model membership disclosure by summed per-iteration gradient disclosures (Theorem~\ref{thm:final-information}), then show that these disclosures increase with GNQ under reasonable assumptions (Theorem~\ref{thm:gaussian-gnq}), yielding a per-datapoint score independent of any attack strategy.
\item We introduce BS-Ghost, which makes auditing \textit{every} datapoint in the training corpus practical at LLM scale through batch-space operations and ghost kernels, without materializing per-datapoint gradients.
\item We connect attack-agnostic GNQ to leakage, memorization, and attribution of training data. GNQ captures the membership vulnerability revealed by LiRA without the need for shadow models and identifies extractable training text more effectively than counterfactual memorization \citep{zhang2023counterfactual}: all 20 highest-GNQ sentences are recovered, versus 7 of 20 selected by counterfactual memorization. For data attribution, GNQ-guided filtering reduces emergent misalignment in Qwen2.5-7B \citep{qwen2025qwen25technicalreport} more than baselines \citep{chen2025personavectorsmonitoringcontrolling}, while attributing 1000 datapoints takes about 17 seconds, roughly $70\times$ faster than both EK-FAC \citep{grosse2023influence} and Concept Influence  \citep{kowal2026concept}.
\end{itemize}
\section{Related Work}
\label{sec:related_work}

\noindent
\emph{Attack-based evaluations of disclosure.}
A large body of work measures disclosure by probing a \emph{single trained model instance} with a particular evaluation or attack family.
For language models, this includes prompting-based extraction and verification pipelines that recover memorized spans or documents
\citep{carlini2021extracting,nasr2023scalable,hayes2025probabilistic,yu2023bagoftricks},
and formalizations connecting memorization to discoverability/extractability of training sequences
\citep{carlini2023quantifying}.
Other empirical probes focus on near-duplicate matching and scalable search procedures used to surface memorized or replicated content
\citep{peng2023nearduplicate,ippolito2023preventing}.
A related line audits leakage of entities or personally identifiable information (PII), where disclosure can occur via associations even when surface form differs
\citep{zhou2023entity,kim2023propile,lukas2023pii}.
In parallel, privacy leakage is commonly operationalized through membership inference attacks (MIAs), which test whether a candidate datapoint was used for training,
including black-box MIAs \citep{shokri2017mia}, language-model-specific variants \citep{mattern2023neighbourhood,xie2024recall},
and strong statistical attacks such as LiRA \citep{carlini2022lira}.
These methods are valuable for demonstrating concrete failure modes, but they are inherently attack/evaluation-specific.

\noindent
\emph{Canary and exposure-style tests.}
Another line of work audits disclosure by inserting synthetic secrets (``canaries'') into the training data and then measuring their leakage.
The exposure metric and related canary-based tests provide a clear operational signal when the audit target is an injected string~\citep{carlini2019secretsharer},
including NLU-focused canary extraction settings~\citep{parikh2022canaryNLU}.
\citet{steinke2023one_run_audit} propose privacy auditing with one training run, again centered on canaries, enabling an efficient single-run audit of leakage for selected secrets.
While these approaches can be carried out in-run, they are not attack-agnostic: the audit procedure is itself an attack protocol. 

\noindent
\emph{Counterfactual memorization and dataset commonality.}
A distinct line of work defines memorization \emph{counterfactually} as the causal effect of including a training example on model behavior, estimated by comparing
\emph{different training runs} on datasets that include versus exclude that example \citep{zhang2023counterfactual}.
This perspective is designed to disentangle memorization of \emph{rare} content from effects driven by \emph{dataset commonality}, in particular
near-duplicates or templated strings that occur across many documents: removing a single instance of highly repeated text should have only a small counterfactual effect.
Importantly, ``common'' in this setting refers to repetition within the \emph{training corpus} (duplication), which is conceptually distinct from ``common knowledge'' in the sense of well-known facts.
Counterfactual memorization adapts to language models the leave-one-out memorization score of
\citep{feldman2020doeslearning} and \citep{feldmanzhang2020whatmemorize}.
However, computing counterfactual memorization requires training many models on different subsets, making it computationally expensive at scale rather than a low-cost, in-run audit (C2). Moreover, it averages over training randomness instead of auditing the particular model and training trajectory produced by the actual run (C3).

\noindent\emph{Information bounds.}
\citet{negrea2019information} and \citet{haghifam2020sharpened} use mutual information as an analysis tool for generalization bounds on noisy SGLD. In contrast, we connect per-iteration gradient disclosure to GNQ for auditing every training datapoint during ordinary mini-batch SGD.

\begin{table}[!t]
\centering
\small
\setlength{\tabcolsep}{4pt}
\renewcommand{\arraystretch}{1.15}
\begin{tabular}{p{0.76\linewidth}cccc}
\toprule
& \multicolumn{4}{c}{\textbf{Desiderata}} \\
\textbf{Prior work} & \textbf{C1} & \textbf{C2} & \textbf{C3} & \textbf{C4} \\
\midrule
\textbf{Attack-based:} \citep{carlini2021extracting,nasr2023scalable,hayes2025probabilistic,yu2023bagoftricks,
carlini2023quantifying,peng2023nearduplicate,ippolito2023preventing,
zhou2023entity,kim2023propile,lukas2023pii,
shokri2017mia,mattern2023neighbourhood}
& \xmark & \xmark & \cmark & \xmark \\
\textbf{Canary-based:} \citep{carlini2019secretsharer,parikh2022canaryNLU,steinke2023one_run_audit}
& \xmark & \xmark & \xmark & \xmark \\
\textbf{Counterfactual:} \citep{zhang2023counterfactual,feldman2020doeslearning,feldmanzhang2020whatmemorize}
& \cmark & \xmark & \xmark & \xmark \\
\bottomrule
\end{tabular}
\caption{Disclosure auditing approaches versus desiderata defined in the Introduction (C1--C4).}
\label{tab:prior_work_desiderata}
\end{table}

\section{Gradient Uniqueness}
\label{sec:gnq}

\begin{wrapfigure}[14]{r}{0.51\columnwidth}
\vspace{-1.5em}
\begin{minipage}{\linewidth}
\color{black}
\captionsetup{type=algorithm}
\begin{algorithm}[H]
\color{black}
\small
\SetAlgoLined
\DontPrintSemicolon
\KwIn{Training set $\mathcal{D}_t$, AI model with parameter vector $\theta \in \mathbb{R}^{P}$, point-wise loss function $\ell[\theta,d]$, learning rate \( \eta \), batch size $B$, number of training iterations $N_r$}
\KwOut{Optimized (final) model parameters \( \theta_{N_r} \)}

Initialize model parameters to $\theta_0$.\; 

Sample mini-batches $\{\mathcal{B}_i\}_{i=0}^{N_r-1}$ from $\mathcal{D}_t$ according to sampling distribution $\{\mathcal{B}_i\} \sim Batch (\mathcal{D}_t)$ with batch-size parameter $B$.\;

\For{$i=0$ {\bfseries to} $N_{r}-1$}{
    $\hat{g}_i = \frac{1}{B} \sum_{d_j \in \mathcal{B}_i} \nabla_{\theta}[\ell[\theta_i, d_j]]$\; 

    $\theta_{i+1} = \theta_i - \eta \cdot \hat{g}_i$\; 

}
\caption{Mini-Batch SGD}
\label{alg:new_sgd}
\end{algorithm}
\end{minipage}
\vspace{0pt}
\end{wrapfigure}

In this section, we give an overview of the statistical model underlying GNQ, the threat model, as well as our central results.

Assume that an attacker wishes to determine whether data point $d_j$ appears in training set $\mathcal D_t$ by analyzing a learned model. We begin with a prior over $\mathcal D_t$. The datapoint $d_j$ appears in $\mathcal D_t$ if and only if $T_j=1$, where $T_j\sim\operatorname{Bernoulli}(\pi)$ and $0<\pi<1$. The other $N_t-1$ training datapoints are drawn independently from a population distribution $\mathcal P$, independently of $T_j$. 
$\mathcal D_t$ is drawn once and then used to train an arbitrary model using mini-batch gradient descent (Alg.~\ref{alg:new_sgd}).

\noindent\textbf{Threat model.} The final model $\theta_{N_r}$ varies with the identities of the sampled training datapoints, target membership, and random batch selection.  The adversary knows the target $d_j$, population distribution $\mathcal P$, and training algorithm, and observes the final model $\theta_{N_r}$.  It does not know the sampled non-target datapoints or batch selections. Given this, it tries to infer $T_j$.

Given this threat model, our goal is to bound the mutual information $I(T_j;\theta_{N_r})$.  This quantity is an attack-agnostic measure of the amount of disclosure regarding the presence of $d_j$ in the data. By the data-processing inequality, any attack based on this observation reveals at most this much membership information~\citep{cover2006elements}. An auditor (model owner) sees the training datapoints and attempts to compute a bound on this mutual information during the unmodified run.

As we argue in the Appendix to the paper, under a reasonable set of assumptions, GNQ can be used to bound this quantity, where GNQ is defined as follows.

\begingroup
\setlength{\topsep}{4pt}
\begin{definition}[Gradient Uniqueness (GNQ)]\label{def:gnq}
Consider training batch $i$. The \emph{gradient uniqueness} of datapoint $d_j$ with respect to batch $i$ is
\begingroup
\setlength{\parskip}{0pt}
\setlength{\abovedisplayskip}{5pt}
\setlength{\belowdisplayskip}{5pt}
\setlength{\abovedisplayshortskip}{0pt}
\setlength{\belowdisplayshortskip}{4pt}
\begin{equation}\label{equ:Dij}
GNQ_{ij}=g_{ij}^{\top}S^{-1}g_{ij},
\end{equation}
\endgroup
where $S=\sum_{\substack{k=1\\k\ne j}}^{\rev{theory}{N_t}}g_{ik}g_{ik}^{\top}+\lambda I_P\in\mathbb R^{P\times P}$ and $g_{ij}=\nabla_\theta\ell(\theta_i,d_j)\in\mathbb R^P$. Here $\lambda>0$ is a small regularizer that makes the matrix invertible, and $I_P$ is the $P\times P$ identity matrix.
\end{definition}
\endgroup

Given this definition, Theorem~\ref{thm:final-information} (Appendix~\ref{app:final-information}) bounds final-model membership disclosure by summed per-iteration gradient disclosures: 
\vspace{-10pt}

\begingroup
\setlength{\parskip}{0pt}
\setlength{\abovedisplayskip}{5pt}
\setlength{\belowdisplayskip}{5pt}
\setlength{\abovedisplayshortskip}{0pt}
\setlength{\belowdisplayshortskip}{4pt}
\begin{equation}\label{eq:final-model-disclosure}
I(T_j;\theta_{N_r})\le\sum_{i=0}^{N_r-1}I(T_j;\widehat g_i\mid\theta_i).
\end{equation}
\endgroup
In this inequality, the per-iteration gradient disclosure $I(T_j;\widehat g_i\mid\theta_i)$ measures the information that the batch gradient
$\widehat g_i$ reveals about $T_j$ beyond the current model $\theta_i$.

\paragraph{The adversary's uncertainty at one iteration.}
Fix $\theta_i=\theta$ and write the update as
\[
\widehat g_i \;=\; W_i+\tfrac1B\,T_jM_{ij}\,g_{ij},\qquad
W_i=\tfrac1B\sum_{k\ne j}M_{ik}\,g_{ik},
\]
where $M_{ik}\in\{0,1\}$ indicates whether datapoint $k$ is sampled into batch $i$
(independently, with probability $q=B/N_t$), and $W_i$ is the contribution of the other
training datapoints. To an adversary who knows $\theta$ and $d_j$, the target's
contribution is a known vector $g_{ij}/B$ ``turned on'' by random variable $T_jM_{ij}$.
The background $W_i$, by contrast, is a sum of gradients of datapoints the adversary
cannot see during an attack, drawn from $\mathcal P$. It is therefore random from the adversary's point
of view and $T_jM_{ij}$ must be
detected as a shift of size $g_{ij}/B$ in the sampled value.  To measure the information available in this signal, we model the adversary's uncertainty
about $W_i$ as Gaussian,
\[
W_i\mid\theta_i=\theta\ \sim\ \mathcal N\!\big(m_i,\ \tfrac{q}{B^2}S\big),
\qquad W_i\perp(T_j,M_{ij})\mid\theta_i=\theta,
\]
with $S$ as in Definition~\ref{def:gnq}: $\tfrac{q}{B^2}S$ is an estimate of the
covariance of $W_i$ obtained from the observed second moment of the non-target
gradients, which the auditor computes in-run. Appendix~\ref{app:gaussian-model} gives the central-limit and
maximum-entropy motivation and states the assumption precisely. Under this predictive
model we have Theorem~\ref{thm:gaussian-gnq} (Appendix~\ref{app:gaussian-model}). Fix $\theta_i=\theta$ and let $\rho_i=\Pr(T_j=1\mid\theta_i=\theta)$, with
$0<\rho_i,q<1$. Then, under the predictive model,
\[
I(T_j;\widehat g_i\mid\theta_i=\theta)
\;=\;I\Big(T_j;\ \sqrt{GNQ_{ij}/q}\;T_jM_{ij}+Z\Big),
\quad Z\sim\mathcal N(0,1)\ \text{independent of }(T_j,M_{ij}),
\]
and the right-hand side is a strictly increasing function of $GNQ_{ij}$.

In other words, once the update is noised using the background covariance, the adversary has access to a one-dimensional Gaussian channel whose signal-to-noise ratio is $GNQ_{ij}/q$.

\paragraph{GNQ bounds disclosure.}
If we write the right-hand side of Theorem~\ref{thm:gaussian-gnq} as a
function $f_i$ of $GNQ_{ij}$ that increases with GNQ\footnote{Note that it also depends on $\rho_i$ and $q$.  However, because $\rho_i$ is unobserved, $f_i$ may be replaced by its supremum over $\rho_i\in[0,1]$, which is still non-decreasing in GNQ and bounds the disclosure at every realization} and then substitute into \eqref{eq:final-model-disclosure}, we have
\begin{equation}\label{eq:gnq-bound}
I(T_j;\theta_{N_r})\;\le\;\sum_{i=0}^{N_r-1}\mathbb E\big[f_i(GNQ_{ij})\big].
\end{equation}
Thus, under the model described, membership disclosure of the released model regarding the inclusion of $d_j$ is bounded by a value that grows with the GNQ scores of $d_j$ along the training trajectory, averaged over the randomness of training. The auditor observes one
trajectory and its GNQ scores, so the sum along that trajectory, over the iterations at which Algorithm \ref{alg:general_ghost} produces a score for $d_j$, is a single sample of the quantity being averaged over, and so it is a natural plug-in estimator.

Note that the bound is
one-sided, so small GNQ certifies small disclosure under the model. Large GNQ means that the data point \emph{might} be at risk of disclosure (developing a two-sided bound would require taking into account possible forgetting, which is beyond the scope of this paper).  Section~\ref{sec:experiments} experimentally examines this issue.

\section{What Does GNQ Measure?}
To show how GNQ measures information disclosure during gradient descent and to give some intuition behind the metric, we consider a simple example---a 2D linear regression model with the squared loss function (Fig.~\ref{fig:Dij_inner}~\subref{fig:data1}). The goal is to audit disclosure by ranking the datapoints by GNQ.

\noindent
\textit{GNQ-based Auditing.} The computation of GNQ can be geometrically represented as the construction of an ellipse summarizing the gradients; $GNQ_{ij}$ is the extent to which the gradient associated with $d_j$ is an outlier with respect to this ellipse, as shown in Fig.~\ref{fig:Dij_inner}~\subref{fig:grad1}. In this figure, we plot the gradients of each of the seven training points, and the associated ellipses. Note that in the definition of $GNQ_{ij}$, the matrix $S$ (and the resulting ellipse) is constructed using all the datapoints, except the point $d_j$ for which $GNQ_{ij}$ is computed. Thus there are two ellipses in Fig.~\ref{fig:Dij_inner}~\subref{fig:grad1}: the blue ellipse is for the case when point 7 is excluded to compute $GNQ_{i7}$, while the red ellipse is for the case when point 7 is included while excluding one of the other six points to compute $GNQ_{ik}$ where $k \in \{ 1,\cdots,6 \}$. Datapoint 7 has a very high GNQ value because it falls outside of the blue ellipse, while all other datapoints have a low GNQ value, as they fall inside of the red ellipse. In Fig.~\ref{fig:Dij_accuracy}, gradients from points 1--6 push the regression line counter-clockwise, roughly around datapoint 5. Datapoint 7 pushes clockwise, explaining its high $GNQ_{i7}$.
\begin{figure}[!t]
  \centering
  \begin{subfigure}[t]{0.33\textwidth}
    \centering
    \includegraphics[width=\textwidth]{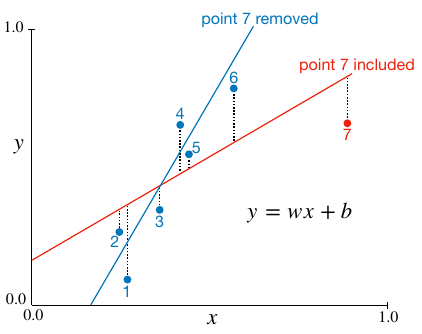}
    \caption{Data space.}
    \label{fig:data1}
    \label{fig:Dij_accuracy}
  \end{subfigure}\hfill
  \begin{subfigure}[t]{0.33\textwidth}
    \centering
    \includegraphics[width=\textwidth]{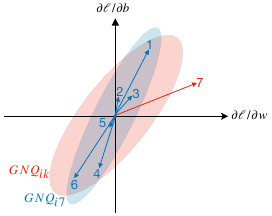}
    \caption{Gradient space.}
    \label{fig:grad1}
  \end{subfigure}\hfill
  \begin{subfigure}[t]{0.33\textwidth}
    \centering
    \includegraphics[width=\textwidth]{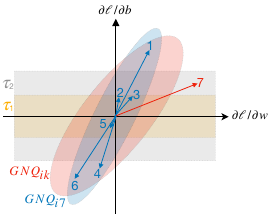}
    \caption{GNQ- vs.\ attack-based auditing.}
    \label{fig:Dij_vs_loss}
  \end{subfigure}
  \caption{What does GNQ Measure?}
  \label{fig:Dij_inner}
  \label{fig:auditing_comparisons}
\vspace{-1em}
\end{figure}

\noindent
\textit{GNQ-based auditing vs attack-based auditing.} 
MIA with a global loss threshold can produce very different rankings from GNQ.  Consider Fig.~\ref{fig:Dij_vs_loss}. In gradient space, the  global loss threshold $\ell=\tau$  appears as a horizontal strip bounded by the two lines $\partial\ell/\partial b=\pm\sqrt{2\tau}$. Points inside the strip $\big(|\partial\ell/\partial b|\le\sqrt{2\tau}\big)$ are classified, by the MIA, as members in the training set; points outside are non-members. We consider two possible thresholds $\tau_1$ and $\tau_2$. If we use $\tau_1$, points 2, 3, and 5 are inside the $\tau_1$-strip, so they will be ranked as high disclosure points, which does not match the GNQ ranking. Moreover, points 1, 4, 6 and 7 are outside the $\tau_1$-strip, so they will be classified safe points; this matches the GNQ ranking regarding points 1, 4, and 6, however it misranks point 7; the most crucial point in terms of information disclosure according to GNQ. Using $\tau_2$, points  2, 3, 5 and 7 are inside the $\tau_2$-strip, so they will be ranked as high disclosure points---which does not match the GNQ ranking except for point 7. Moreover, points 1, 4 and 6 are outside the $\tau_2$-strip, so their low-disclosure classification agrees with GNQ. Further, any ranking by loss alone depends solely on the residual (vertical distance to the fit), so points \(2,3,\) and \(5\)---having smaller residuals (lower loss) than point \(7\)---are necessarily ranked as higher disclosure for any choice of threshold. By contrast, GNQ, uses the full gradient geometry: it accounts for (1) the residual $r$, (2) the feature vector ($[x,1]$ in our 2D example)---both via the gradient formula $g=-r[x,1]$---and (3) inter-example correlations via \(S^{-1}\). 

\section{Implementation: LLM-Scale GNQ via BS-Ghost}
\label{sec:ghostgnq}
The most direct implementation of GNQ follows Definition~\ref{def:gnq} almost literally: forming the $P\times P$ leave-one-out gradient second-moment matrices and inverting. This is impractical, as $P$ can be in the billions.  We show that it \emph{is} practical to implement GNQ on a per-batch basis
using standard matrix identities (push-through) together with a leave-one-out conversion (Sherman-Morrison). This removes the cubic dependence upon $P$. 
We also show how to compute the required batch-space quantities without explicitly forming the per-example gradient vectors (i.e. no extra backward passes required). A detailed time and memory complexity analysis is given in Appendix~\ref{app:bsghost_complexity}.

We now describe the three key ideas that enable the efficient implementation in more detail.
\Needspace{.48\textheight}
\begin{wrapfigure}{r}{0.50\columnwidth}
\vspace{-1.3em}
\begin{minipage}{\linewidth}
\captionsetup{type=algorithm}
\begin{algorithm}[H]
\small
\caption{BS-Ghost GNQ}
\label{alg:general_ghost}
\KwIn{Model at iteration $i$ ($\theta_i$), training batch $\mathcal{B}_i$ (size $B$), $\lambda>0$}
\KwOut{$\mathrm{GNQ}_i\in\mathbb{R}_{\ge 0}^B$ : a vector containing all $GNQ_{ij}$ values}
\BlankLine
\tcp{\textbf{0. Hook setup (one-time)}}
Register forward hooks to cache layer inputs $X^{(\ell)}$ and backward hooks to cache output errors $\delta^{(\ell)}$\;
\BlankLine
\tcp{\textbf{1. No extra backprops}}
Clear training caches\;
Run the standard forward/backward on $\mathcal{B}_i$ (as required by SGD)\;
Hooks populate $\{X^{(\ell)},\delta^{(\ell)}\}_{\ell=1}^{L}$\;
\BlankLine
\tcp{\textbf{2. Build batch kernel $K^{(i)} \in \mathbb{R}^{B \times B}$}}
$K^{(i)} \leftarrow 0_{B\times B}$\;
\For{each instrumented layer $\ell$}{
  Compute $K^{(\ell)}$ (e.g., via \eqref{eq:K_linear_bias})\;
  $K^{(i)} \leftarrow K^{(i)} + K^{(\ell)}$\;
}
\BlankLine
\tcp{\textbf{3. Solve in batch space}}
Solve $M$ from $\big(K^{(i)}+\lambda I_B\big) M = K^{(i)}$\;
$h^{(i)} \gets \mathrm{diag}(M)$\;
$\mathrm{GNQ}_i \leftarrow h^{(i)}\oslash(1-h^{(i)})$\tcp*{Lemma~\ref{lem:loo_gnq}}
\end{algorithm}
\end{minipage}
\vspace{-1.2em}
\end{wrapfigure}
\noindent
\textbf{Step (1) : $B$ per-example gradients instead of ${N_t}$.}
Forming $S_{-j}^{(i)}$ requires ${N_t}$ separate backward passes to compute per-example gradients. However, the mini-batch SGD algorithm itself uses the mini-batch gradient $\hat{g}_i = \frac{1}{B} \sum_{d_j \in \mathcal{B}_i} \nabla_{\theta}[\ell[\theta_i, d_j]]$ instead of the dataset gradient $g_i = \frac{1}{N_t}\sum_{j=1}^{N_t} \nabla_{\theta}[\ell[\theta_i, d_j]]$. Hence, we compute $GNQ_{ij}$ using $S_{-j}^{(i)} = \sum_{\substack{d_k \in \mathcal{B}_i \\ k \neq j}} g_{ik}g_{ik}^{\top} +\lambda I_P$ instead of $S_{-j}^{(i)} = \sum_{\substack{k=1\\k\ne j}}^{N_t} g_{ik}g_{ik}^{\top} +\lambda I_P$. This reduces the number of per-example gradient vectors computed and materialized at each iteration. Under the population model of Section~\ref{sec:gnq}, both sums estimate the same second moment of the non-target gradients, the batch sum with $B$ draws in place of $N_t-1$.

\noindent
\textbf{Step (2) : Batch-Space Formulation.}
Fix an SGD step $i$ with batch $\mathcal{B}_i$ (size $B$), parameters $\theta_i$, and per-example gradients
$\{g_{ij}\}_{j\in\mathcal{B}_i}$.
Let $G\in\mathbb{R}^{B\times P}$ be the matrix whose $j$-th row is $g_{ij}^{\top}$.
Define the regularized ``total'' matrix in parameter space and its batch-space Gram matrix:
\par\WFclear
\noindent\(
S_{\mathrm{tot}}^{(i)}\triangleq
G^{\top}G+\lambda I_P\in\mathbb R^{P\times P},
\qquad
K^{(i)}\triangleq GG^{\top}\in\mathbb R^{B\times B}.
\)\par
For $\lambda>0$, the push-through identity gives
\((G^{\top}G+\lambda I_P)^{-1}G^{\top}
=G^{\top}(GG^{\top}+\lambda I_B)^{-1}\).\par
\begingroup
\setlength{\parskip}{0pt}
\setlength{\jot}{1pt}
For each $j\in\mathcal B_i$, define $h_j^{(i)}\triangleq g_{ij}^{\top}(S_{\mathrm{tot}}^{(i)})^{-1}g_{ij}$. With $g_{ij}=G^{\top}e_j$ ($e_j$ the $j$th standard basis vector in $\mathbb R^B$), applying this identity gives
\begin{align*}
h_j^{(i)}
&= e_j^{\top} G (G^{\top} G+\lambda I_P)^{-1} G^{\top} e_j \\
&= e_j^{\top} GG^{\top} (GG^{\top}+\lambda I_B)^{-1} e_j \\
&= \big[\,K^{(i)}\big(K^{(i)}+\lambda I_B\big)^{-1}\big]_{jj}.
\end{align*}
\endgroup
Thus $h^{(i)}$ is computable using only $B\times B$ matrices. Since $S_{\mathrm{tot}}^{(i)}=S_{-j}^{(i)}+g_{ij}g_{ij}^{\top}$, Sherman--Morrison gives the following lemma (proof in Appendix~\ref{app:batch-algebra}).
\begin{lemma}\label{lem:loo_gnq}
Whenever $0\le h_j^{(i)}<1$, $GNQ_{ij}=h_j^{(i)}/(1-h_j^{(i)})$.
\vspace{-1em}
\end{lemma}

A single solve with $K^{(i)}+\lambda I_B$ yields GNQ for every datapoint in the batch, allowing the audit to cover the training corpus as its batches are processed.

\noindent
\textbf{Step (3) : Ghost Kernel Construction.}
\noindent The use of specialized ``ghost kernels'' enables practical computation of GNQ. 
Instead of materializing gradients, each ghost kernel reuses backpropagation results to compute a Gram matrix describing correlations between datapoints' gradients.

Note that to compute GNQ, we need a layer-specific ghost kernel for each instrumented layer (whether it is linear, convolutional, embedding, or another type). If not all trainable parameters are instrumented, the reported scores are lowered, and the argument that GNQ upper-bounds disclosure no longer applies. However, if the goal is simply to order training points by disclosure, such scores may still be useful. The final, comprehensive relationship matrix for the entire batch is simply the sum of these individual layer-wise kernels.

The mathematical trick of bypassing explicit gradient materialization was explored in ``ghost clipping'' used in differential privacy and the ``ghost dot-product'' used in data valuation~\citep{lee2020scalingdifferentiallyprivatedeep, li2022large, wang2025datashapleytrainingrun}. Our approach uses the ghost trick to construct the complete matrix in batch space for the leave-one-out inversion required by GNQ.

We now describe the required ghost kernel for a linear layer.  Other ghost kernels can be implemented similarly---we describe the ghost kernels of the embedding and convolution layers in Appendix~\ref{app:ghost-kernels}.
Consider a single-position linear layer $\ell$ with cached activations
$X^{(\ell)}\in\mathbb{R}^{B\times n_{\mathrm{in}}^{(\ell)}}$ and output backprop errors
$\delta^{(\ell)}\in\mathbb{R}^{B\times n_{\mathrm{out}}^{(\ell)}}$ from the training backward pass at step $i$.
For sample $j$, the per-example weight gradient is
$g_{j}^{(\ell)} = \delta^{(\ell)}_j {X^{(\ell)}_j}^{\top}$.
Therefore, the layerwise contribution to the batch kernel is
\begin{equation}
\label{eq:K_linear_bias}
K^{(\ell)} \;=\; \big(\delta^{(\ell)}{\delta^{(\ell)}}^{\top}\big)\odot\big(X^{(\ell)}{X^{(\ell)}}^{\top}+\mathbf{1}\mathbf{1}^{\top}\big).
\end{equation}
Summing across instrumented layers yields $K^{(i)}=\sum_{\ell=1}^{L}K^{(\ell)}$. 

\setlength{\parskip}{4pt}
\section{Experiments}
\label{sec:experiments}

\begin{table}[!t]
\centering
\captionsetup{skip=4pt}
\begin{minipage}[t]{0.5\linewidth}
\centering\small
\setlength{\tabcolsep}{4pt}
\begin{tabular}{@{}lrr@{}}\toprule
\multicolumn{3}{c}{\textbf{GPT-2 on WikiText-2}}\\\midrule
Method & Sec/iteration & Tokens/sec\\
GPT-2 & 0.53 & 3864\\
+ BS-Ghost & 0.59 & 3471\\\bottomrule
\end{tabular}
\captionof{table}{GPT-2 training cost w/ and w/o BS-Ghost.}
\label{tab:gpt2-ghost-overhead}
\end{minipage}\hfill
\begin{minipage}[t]{0.5\linewidth}
\centering\small
\setlength{\tabcolsep}{4pt}
\begin{tabular}{@{}lrrrr@{}}\toprule
\multicolumn{5}{c}{\textbf{Small Multi-Layer Perceptron}}\\\midrule
Method & Extract (s) & GNQ (s) & Total (s) & MB\\
Naive & 0.026 & 5.44 & 5.47 & 914\\
BS-Ghost & 0.00 & 0.04 & 0.04 & 0.1\\\bottomrule
\end{tabular}
\captionof{table}{Naive vs BS-Ghost on the MLP.}
\label{tab:mlp-naive-vs-ghost}
\end{minipage}
\vspace{-1em}
\end{table}

\begin{figure}[!t]
\centering
\captionsetup{skip=4pt}
\begin{minipage}[t]{0.49\linewidth}
\centering
\includegraphics[width=.88\linewidth]{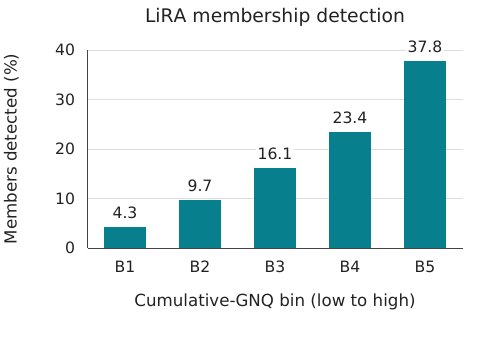}
\captionof{figure}{GNQ strongly predicts MIA success.}
\label{fig:lira-detection}
\end{minipage}\hfill
\begin{minipage}[t]{0.49\linewidth}
\centering
\includegraphics[width=.88\linewidth]{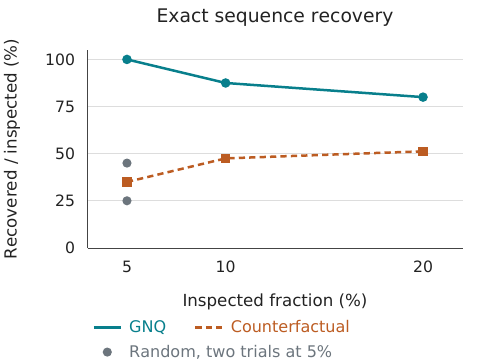}
\captionof{figure}{GNQ identifies extractable training text.}
\label{fig:text-extraction}
\end{minipage}
\vspace{-1em}
\end{figure}

\noindent\textbf{Correctness Of BS-Ghost.}
We use a small MLP to make naive GNQ's $P\times P$ inversions feasible. BS-Ghost reproduces its scores to numerical precision: $\max_j|\mathrm{GNQ}^{\text{Ghost}}_j-\mathrm{GNQ}^{\text{Naive}}_j|\approx2\times10^{-10}$ for the same batch, parameters, and regularizer. Each naive matrix has about 20 million entries. This agreement checks the batch-space reformulation and ghost kernels against the GNQ definition.

\noindent\textbf{Efficiency Of BS-Ghost.}
We benchmark GPT-2 Small on WikiText-2 ($B=16$, sequence length 128) and compare BS-Ghost with naive GNQ on a $100$--$40$--$10$ MLP trained on MNIST ($B=32$), using a 16GB Tesla P100-SXM2 GPU---compared to more recent GPUs, the amount of memory and compute available on this GPU is very limited, making it an appropriate hardware platform to test efficiency. BS-Ghost adds about 11\% to GPT-2 training time and accelerates MLP GNQ computation over $100\times$, while reducing peak memory from 914 to 0.1 MB (Tables~\ref{tab:gpt2-ghost-overhead} and~\ref{tab:mlp-naive-vs-ghost}). The extra peak memory for GPT-2 Small and Medium is below 0.1 MiB (Table~\ref{tab:added-peak-memory}, Appendix~\ref{app:efficiency-setup}). Attributing 1000 Qwen2.5-7B datapoints takes about 17 seconds, roughly $70\times$ faster than EK-FAC (Figure~\ref{fig:qwen-runtime}; Appendix~\ref{app:influence-comparison}).\looseness=-1

\begin{figure}[!t]
\centering
\captionsetup{skip=4pt}
\begin{minipage}[t]{.49\linewidth}
\centering\vspace{0pt}
\begin{subfigure}[t]{.32\linewidth}
\includegraphics[width=.88\linewidth,trim=0 0 320 24,clip]{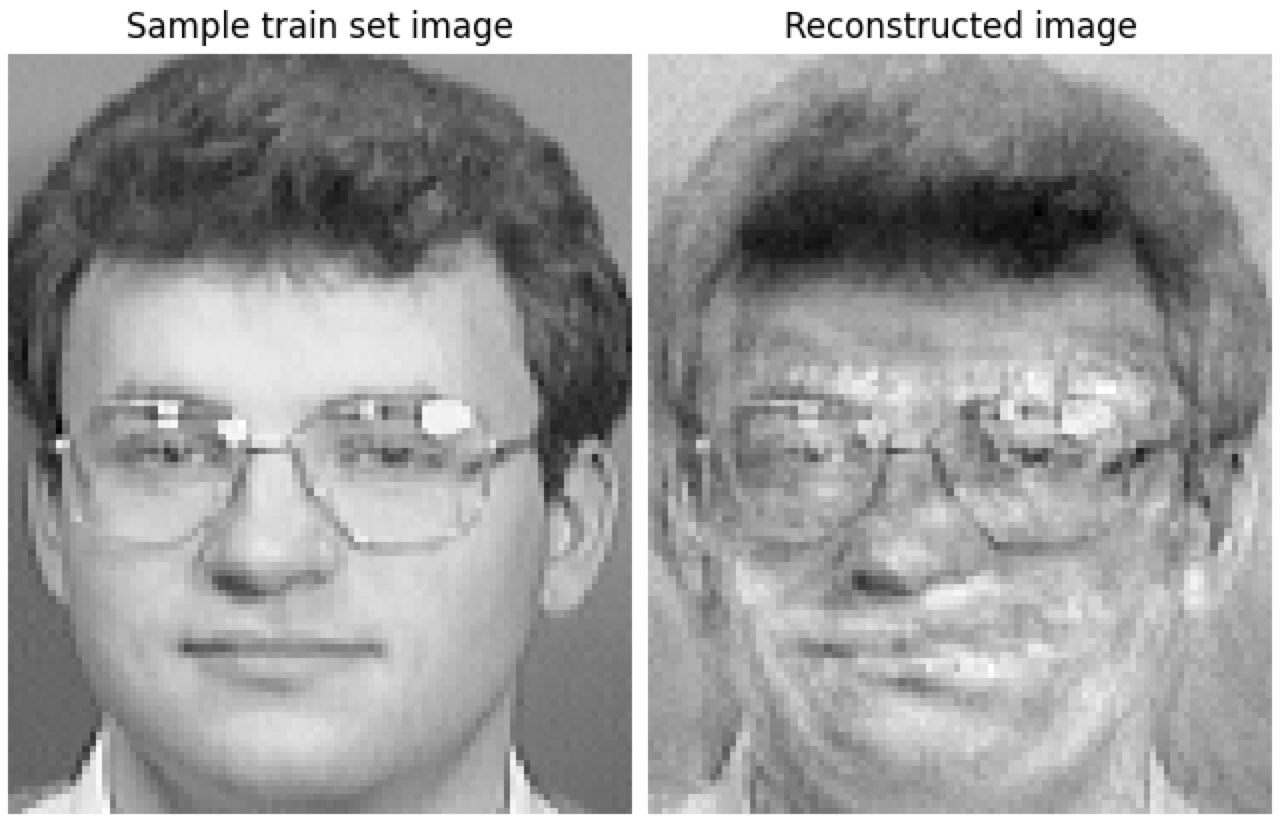}
\caption{\rev{plain}{Training.}}
\end{subfigure}\hfill
\begin{subfigure}[t]{.32\linewidth}
\includegraphics[width=.88\linewidth,trim=320 0 0 24,clip]{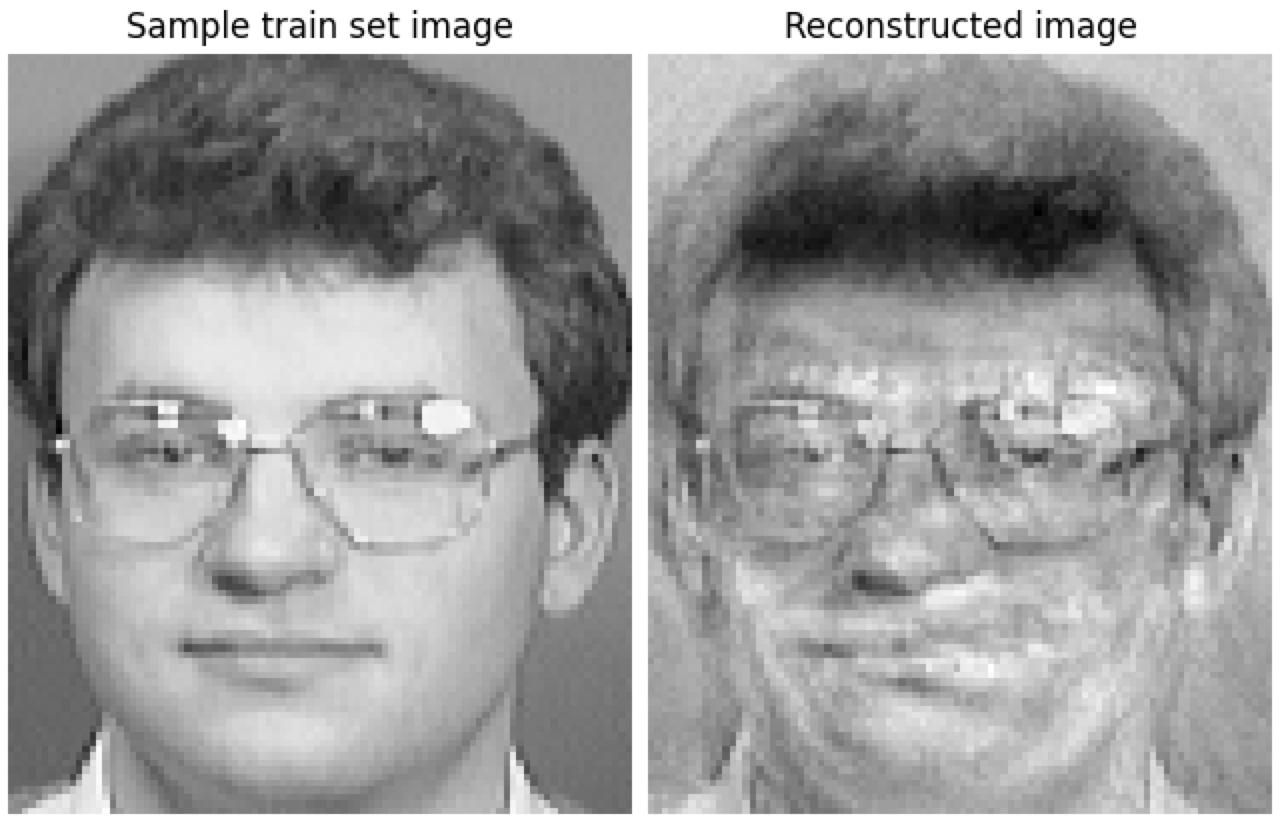}
\caption{\rev{plain}{Before.}}
\end{subfigure}\hfill
\begin{subfigure}[t]{.32\linewidth}
\includegraphics[width=.88\linewidth,trim=211 0 0 16,clip]{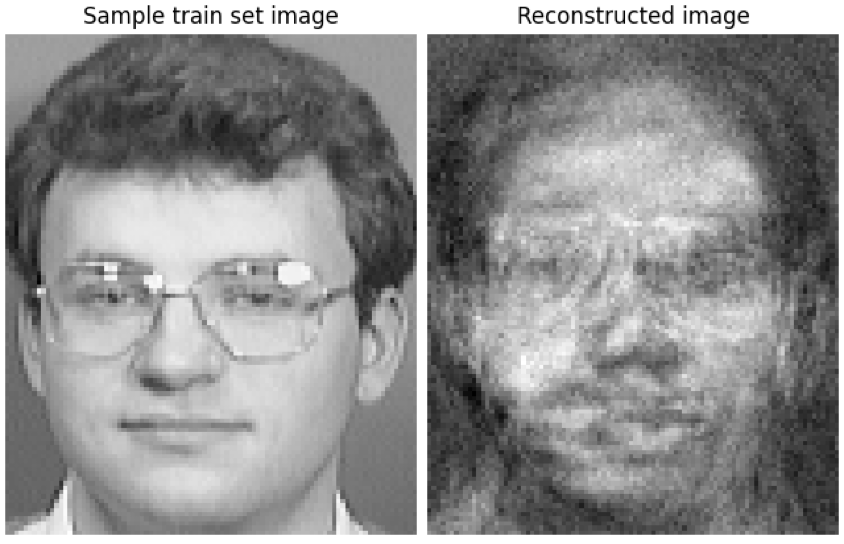}
\caption{\rev{plain}{After.}}
\end{subfigure}
\par\vspace{10.81pt}
\caption{GNQ-based filtering on Faces/MLP: removing the single highest-GNQ training datapoint in the target class makes reconstruction fail.}
\label{fig:reconstruction-examples}
\end{minipage}\hfill
\begin{minipage}[t]{.49\linewidth}
\centering\vspace{0pt}
\includegraphics[width=.88\linewidth]{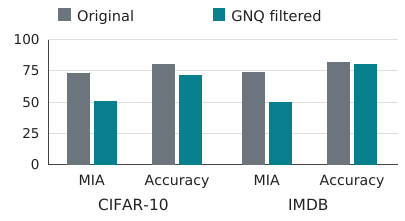}
\captionof{figure}{GNQ-based filtering reduces MIA success, with less accuracy loss on IMDB. Both metrics are percentages; lower MIA is better.}
\label{fig:filtering}
\end{minipage}
\vspace{-1em}
\end{figure}

\begin{figure}[!t]
\centering
\captionsetup{skip=4pt}
\begin{minipage}[t]{0.49\linewidth}
  \centering
  \includegraphics[width=.88\linewidth]{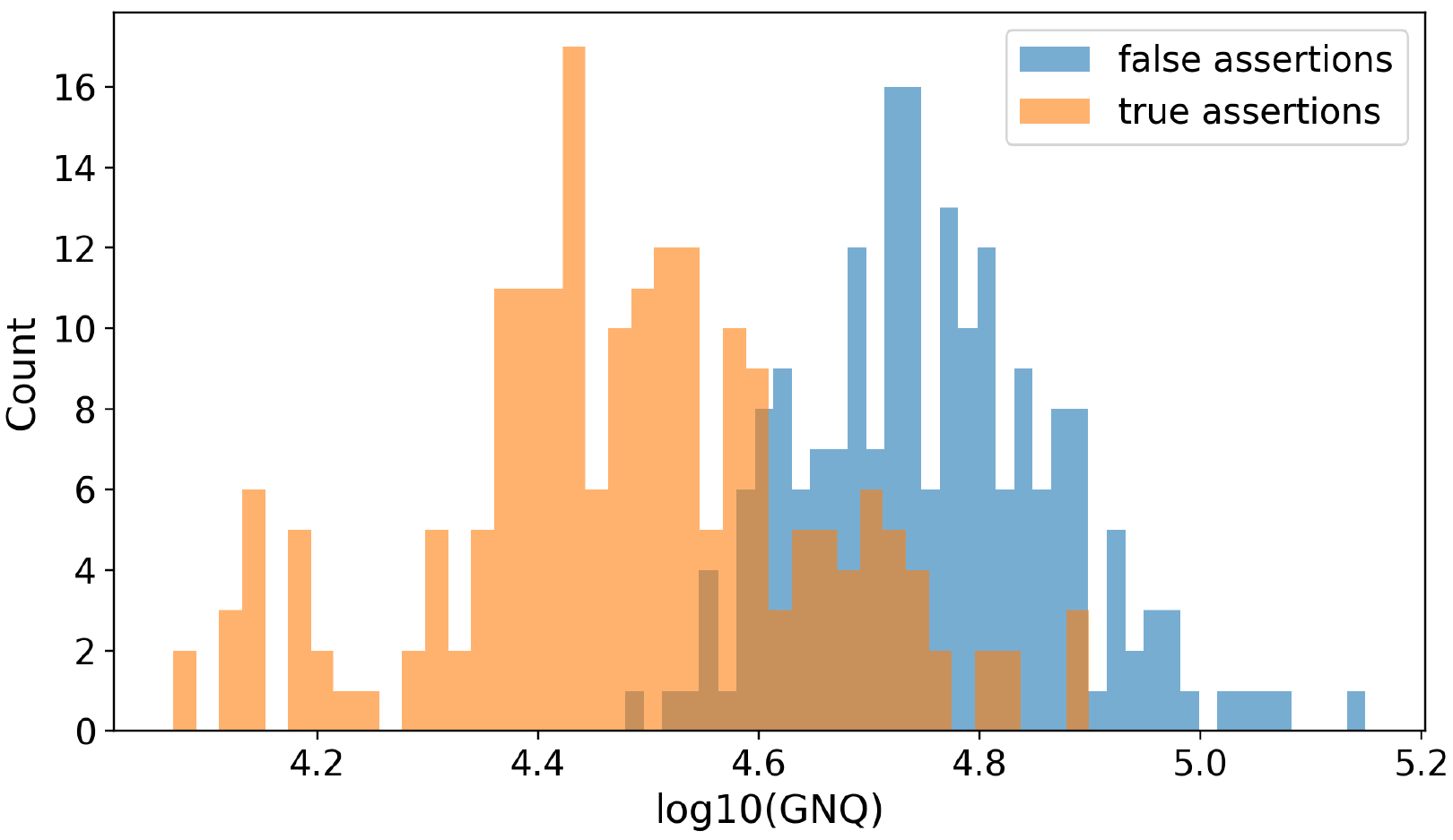}
\captionof{figure}{GNQ assigns lower memorization  to
\textcolor[HTML]{FFB26E}{common knowledge} than to
\textcolor[HTML]{78ADD2}{surprising assertions}.}
  \label{fig:gnq_common_knowledge}
\end{minipage}
\hfill
\begin{minipage}[t]{0.49\linewidth}
  \centering
  \includegraphics[width=.88\linewidth]{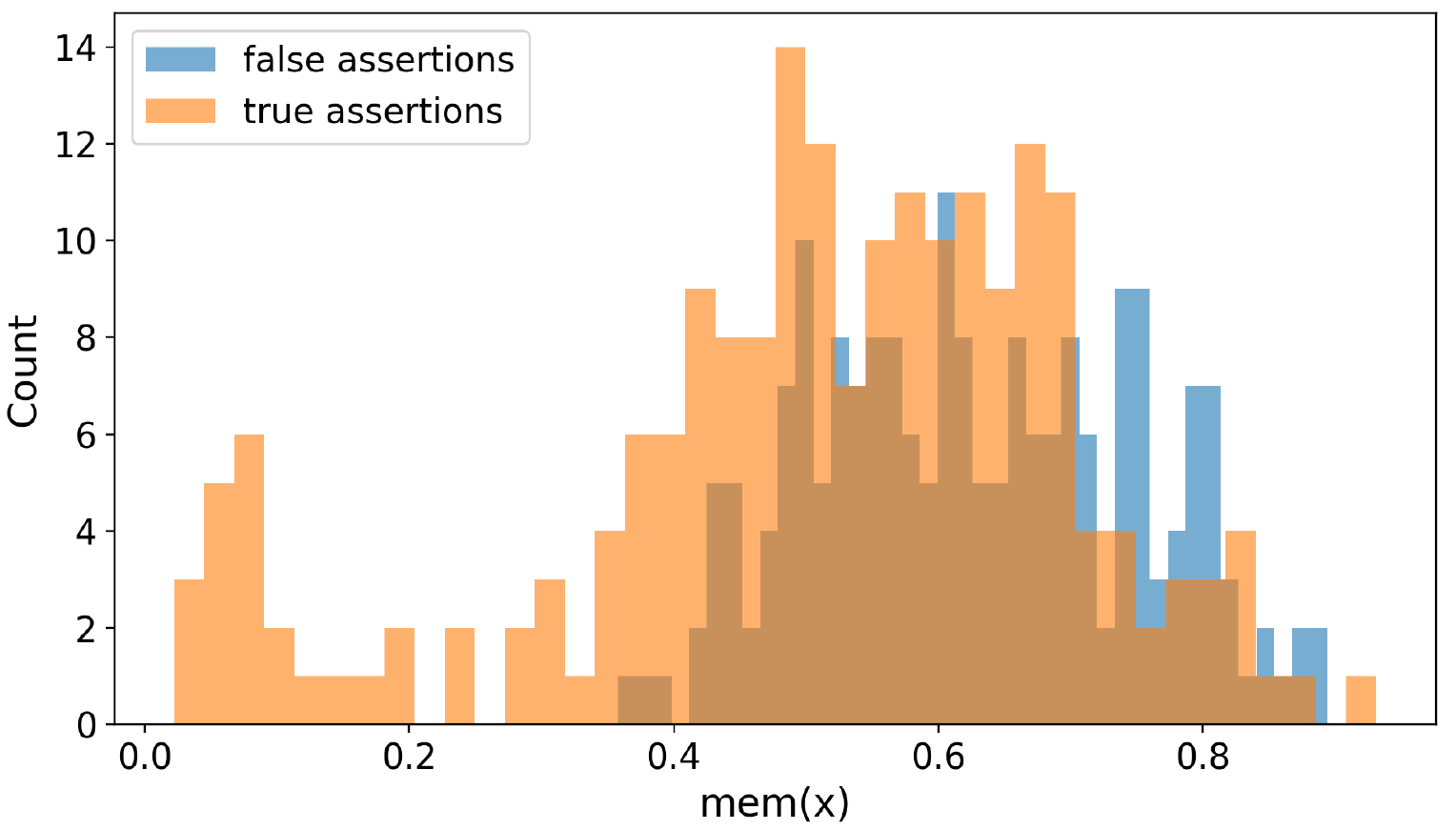}
  \captionof{figure}{In contrast, counterfactual assigns comparable memorization scores to both groups.}
  \label{fig:counterfactual}
\end{minipage}
\vspace{-1em}
\end{figure}

\noindent\textbf{Validating GNQ As A Disclosure Measure.} Membership and reconstruction attacks test the leakage identified by GNQ; removal and retraining test its causal connection.\par

\noindent\emph{Membership Disclosure.} We test GNQ's information-theoretic interpretation by comparing GNQ accumulated during training with LiRA's detection of membership $T_j$ from the final model. We use GPT-2 on WikiText-103, as in the original LiRA evaluation~\citep{carlini2022lira}. At a false-positive rate of $0.1\%$, LiRA's true-positive rate rises across all five cumulative-GNQ bins, from $4.3\%$ in the lowest to $37.8\%$ in the highest (Figure~\ref{fig:lira-detection}). GNQ matches this strong MIA's auditing conclusions without the need for shadow models, while remaining attack-agnostic.

\noindent\emph{Disclosure Beyond Membership.}
We test content disclosure through class-level model inversion~\citep{fredrikson2015modelinversion} on Faces/MLP and MNIST/ResNet (illustrated in Figure~\ref{fig:reconstruction-examples}(b)), summing GNQ over each class's training datapoints and iterations. Higher-GNQ classes yield reconstructions closer to same-class training images in penultimate-layer features, with correlations of $-0.90$ and $-0.96$ between class GNQ and reconstruction distance (Figure~\ref{fig:class-reconstruction}(a,b)). The lowest-GNQ Faces class reconstructs poorly; the highest remains recognizable after one epoch at 10--30\% test accuracy, showing disclosure despite underfitting (Figure~\ref{fig:class-reconstruction}(c,d), Appendix~\ref{app:class-reconstruction}). Agreement across attacks supports GNQ's principled, attack-agnostic design.

\noindent\emph{From Correlation To Causation.}
Removal and retraining establish a causal connection between GNQ and disclosure. For Faces/MLP, removing the single highest-GNQ training datapoint in the target class makes its reconstruction fail (Figure~\ref{fig:reconstruction-examples}). Removing the highest-GNQ 10\% on CIFAR-10/ResNet and 5\% on IMDB/BERT reduces MIA success from about 73\% to 51\% and from 74\% to 50\%, respectively, with accuracy losses of about 9.5 and only 1.5 percentage points (Figure~\ref{fig:filtering}). At 50\%, the adversary predicts membership at chance, so the MIA fails.

\noindent\textbf{Validating GNQ As A Memorization Measure.} We compare two attack-agnostic measures, GNQ and counterfactual memorization~\citep{zhang2023counterfactual}, on fidelity and practicality.

\noindent\emph{Text Extraction.}
We collect GNQ while fine-tuning GPT-2 on 200 common-knowledge sentences and 200 surprising false assertions constructed with an LLM. We attempt exact recovery from a 12-token prefix~\citep{carlini2023quantifying}, selecting targets by maximum GNQ, counterfactual memorization, or randomly. This tests whether each score concentrates extractable sentences among its highest-ranked targets, where an auditor would inspect first. All 20 highest-GNQ sentences are recovered, versus 7 counterfactual targets and 5 or 9 in two random trials (Figure~\ref{fig:text-extraction}). Of 80 targets, GNQ recovers 64 versus 41 for counterfactual memorization (Appendix~\ref{app:text-extraction}, Table~\ref{tab:ghostgnq-targeted-extraction}).

\noindent\emph{Common Knowledge.}
Common facts can be reproduced without memorizing the particular training sentence (C4). They receive lower GNQ than surprising assertions (Figure~\ref{fig:gnq_common_knowledge}), consistent with less datapoint-specific memorization, while counterfactual scores overlap (Figure~\ref{fig:counterfactual}). This distinction matters because recovering a familiar fact need not reveal which training sentence supplied it. 

GNQ offers greater fidelity---better text recovery and common-knowledge discrimination---and greater practicality: BS-Ghost is in the same run, while counterfactual requires training many models.\looseness=-1

\begin{figure}[!t]
\centering
\captionsetup{skip=4pt}
\begin{minipage}[t]{.49\linewidth}
\centering\vspace{0pt}
\includegraphics[width=.88\linewidth]{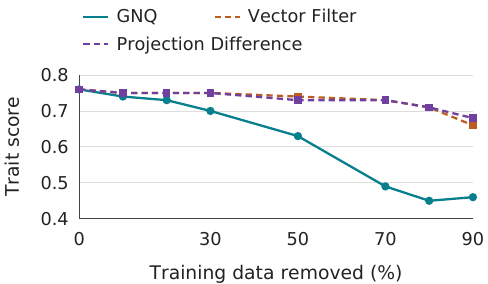}
\captionof{figure}{Qwen2.5-7B Bad Medical Advice after removal and retraining. Lower is better.}
\label{fig:qwen-medical}
\end{minipage}\hfill
\begin{minipage}[t]{.49\linewidth}
\centering\vspace{0pt}
\includegraphics[width=.88\linewidth]{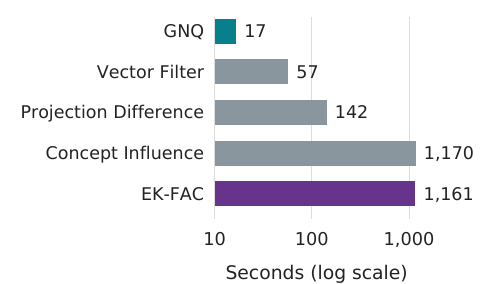}
\captionof{figure}{Attributing 1000 datapoints; baseline times from \citet{kowal2026concept}.}
\label{fig:qwen-runtime}
\end{minipage}
\vspace{-1em}
\end{figure}

\noindent\textbf{Connecting Disclosure To Training Data Attribution.}
We compute GNQ for Qwen2.5-7B using datapoint gradients projected onto each trait's behavior vector, following Concept Influence~\citep{kowal2026concept}. The projection focuses the audit on gradient directions associated with that behavior. GNQ-guided removal and retraining reduce Bad Medical Advice more than either baseline at every tested nonzero fraction (Figure~\ref{fig:qwen-medical}). GNQ's lower trait scores demonstrate that its ranking better identifies the datapoints contributing to the targeted behavior. Misaligned Opinions follows the same pattern (Appendix~\ref{app:qwen-removal}, Figure~\ref{fig:qwen-removal}). Future work can deepen this initial connection to attribution.

The experiments above connect GNQ to leakage, memorization, and data attribution. These connections are natural: all three depend on training data, and disclosure captures information about those data in the final model. Selecting the training data is one training choice; we next examine how model and optimization choices affect disclosure.


\begin{figure}[!t]
\centering
\captionsetup{skip=4pt}
\begin{minipage}[t]{0.49\linewidth}
  \centering
  \includegraphics[width=.88\linewidth]{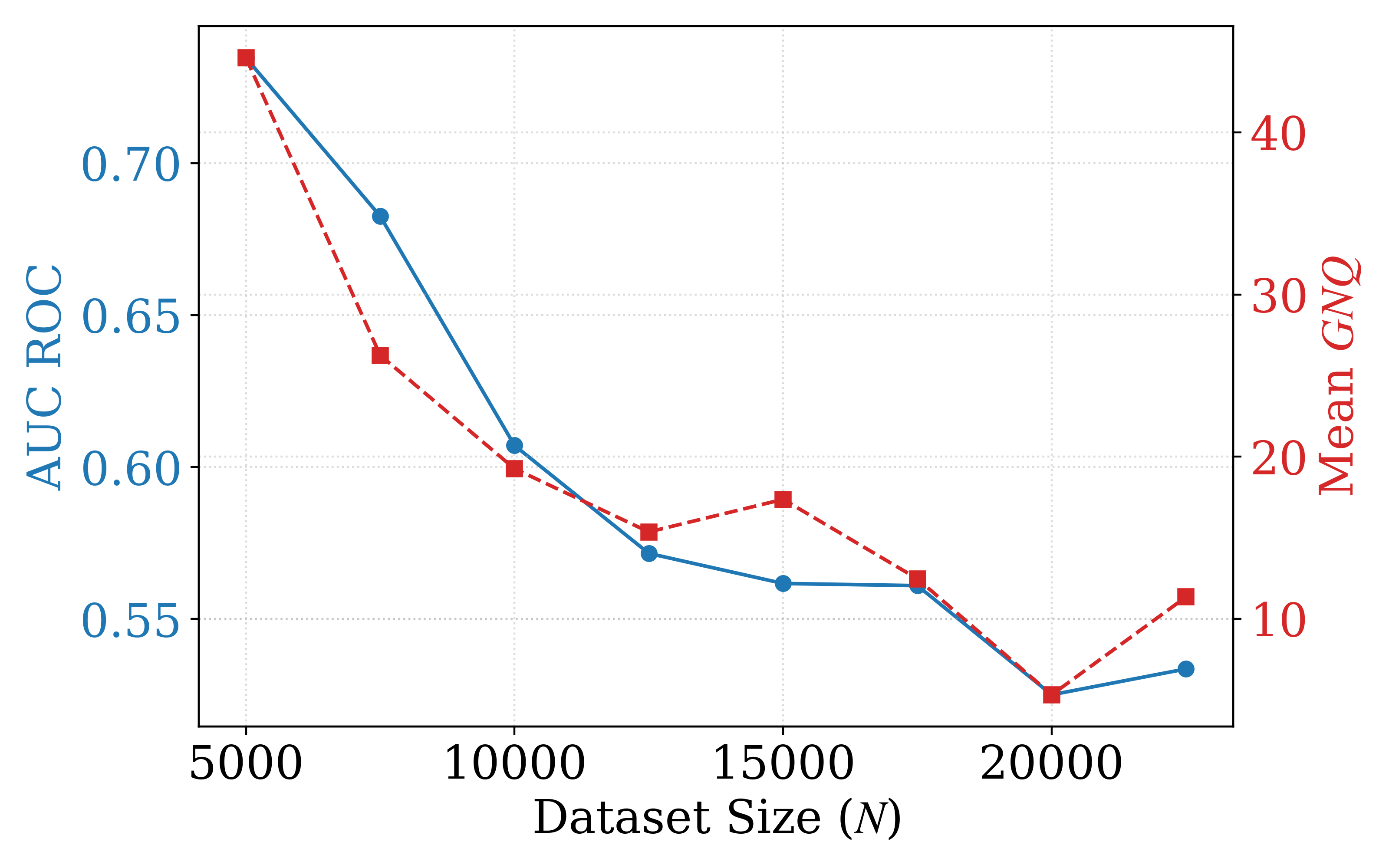}
  \captionof{figure}{Dataset size versus GNQ and MIA AUC-ROC on CIFAR-10/CNN.}
  \label{fig:dataset-size}
\end{minipage}
\hfill
\begin{minipage}[t]{0.49\linewidth}
  \centering
  \includegraphics[width=.88\linewidth,trim=0 0 0 6bp,clip]{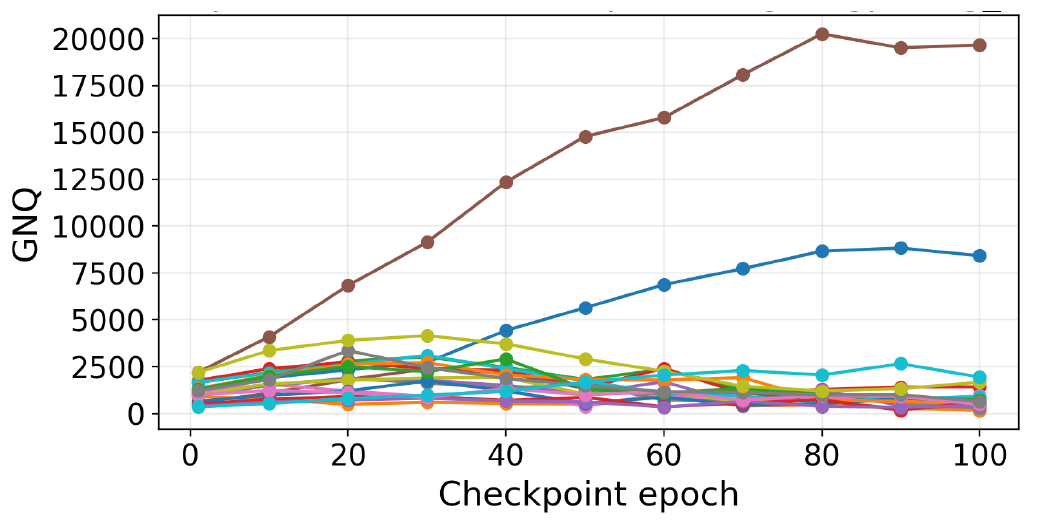}
  \captionof{figure}{GNQ trajectories for 20 randomly selected AG-News training datapoints.}
  \label{fig:gnq_ckpt_trajectories}
\end{minipage}
\vspace{-1em}
\end{figure}

\noindent\textbf{GNQ Explains How \textit{Other} Training Choices Affect Disclosure.}
On CIFAR-10/CNN, increasing the dataset from 5000 to 22500 datapoints reduces mean GNQ (averaging per-datapoint maxima) fourfold, while MIA AUC-ROC falls from about 0.74 to 0.53, close to chance (Figure~\ref{fig:dataset-size}). GNQ explains this trend: additional datapoints can supply similar gradients, making an individual contribution less distinguishable, hence lowering GNQ. Further sweeps show GNQ and MIA AUC-ROC rising overall with iterations, model size, and batch size, and falling with learning rate (Appendix~\ref{app:training-choices}, Figure~\ref{fig:training-sweeps}).\looseness=-1

\noindent\textbf{GNQ Captures How Disclosure Emerges During Training.}
We audit 500 AG-News datapoints at 11 checkpoints of GPT-2 trained for 100 epochs, keeping the batch partition fixed (Appendix~\ref{app:trajectory-setup}). This holds neighboring training datapoints constant, so changing GNQ reflects the model's evolving gradient geometry. Some GNQ trajectories remain low; others separate progressively and stay high (Figure~\ref{fig:gnq_ckpt_trajectories}). 
Tracking GNQ reveals disclosure differences within the same topic: two technology articles differ by about $24\times$ in average GNQ (Table~\ref{tab:gnq_traj_texts}).
These GNQ trajectories motivate future work on training-time interventions targeting datapoints and stages where disclosure emerges.

\section{Conclusion}
Auditing \textit{every} training datapoint makes it possible to locate concentrations of disclosure and examine how training shapes them. GNQ and BS-Ghost make this practical at LLM scale. This enables data selection and training choices informed by both disclosure and model utility, while opening the way to training-time interventions targeting the datapoints and stages where disclosure emerges.



\clearpage
\setlength{\parskip}{6pt}


\bibliographystyle{iclr2027_conference}
{\color{black}\bibliography{my}}

\clearpage
\appendix
\section{From Final-Model Disclosure to Gradient Uniqueness}\label{sec:theory}

\textbf{Problem setup.} Fix a candidate training datapoint $d_j$ and a membership prior $T_j\sim\operatorname{Ber}(\pi)$, where $0<\pi<1$. Draw the other $N_t-1$ datapoints independently from a population distribution $\mathcal P$, independently of $T_j$. They are drawn once and reused throughout training; $d_j$ is added when $T_j=1$. Thus $N_t$ denotes the size of the target-present training set. The adversary knows $d_j$, $\mathcal P$, and the training algorithm, but not the realized non-target datapoints. These quantities are fixed background knowledge in the information expressions below.

For the local sampling analysis, let $M_{ik}\sim\operatorname{Ber}(q)$, $q=B/N_t$, be fresh, mutually independent batch masks, independent of the training data, initialization, and earlier masks. We assume $0<B<N_t$ and retain the fixed normalizer $B$ in Algorithm~\ref{alg:new_sgd}. Index the target-present sample by $1,\ldots,N_t$, with target index $j$, and write
\[
g_{ik}:=\nabla_\theta\ell(\theta_i,d_k),\qquad
\widehat g_i=\frac1B\sum_{\substack{k=1\\k\ne j}}^{N_t}M_{ik}g_{ik}
+\frac1B T_jM_{ij}g_{ij}.
\]
The SGD update is $\theta_{i+1}=\theta_i-\eta\widehat g_i$, with data-independent initialization $\theta_0$. Although the auditor knows the realized gradients, the adversary does not know the datapoints determining the non-target gradients. Natural logarithms are used throughout, so information is measured in nats.

\subsection{From Final-Model Disclosure to Per-Iteration Gradient Disclosure}\label{app:final-information}

\begin{theorem}[Final-model membership disclosure]\label{thm:final-information}
For mini-batch SGD with data-independent initialization, final-model membership disclosure is bounded by the sum of per-iteration gradient disclosures:
\[
I(T_j;\theta_{N_r})\le
\sum_{i=0}^{N_r-1} I(T_j;\widehat g_i\mid\theta_i).
\]
This statement does not require Gaussian gradients or conditional independence between the training datapoints.
\end{theorem}
\begin{proof}
The next model is computed from $\theta_i$ and $\widehat g_i$, so observing this pair reveals at least as much as observing $\theta_{i+1}$ alone. The chain rule splits the pair's information into what the current model reveals and what the gradient adds:
\begin{align*}
I(T_j;\theta_{i+1})
&\le I(T_j;\theta_i,\widehat g_i)\\
&=I(T_j;\theta_i)+I(T_j;\widehat g_i\mid\theta_i).
\end{align*}
Summing over $i=0,\ldots,N_r-1$ cancels the intermediate model-information terms: each appears once on each side. The initialization contributes $I(T_j;\theta_0)=0$, leaving final-model membership disclosure bounded by the sum of per-iteration gradient disclosures.
\end{proof}

\subsection{Per-Iteration Gradient Disclosure and GNQ}\label{app:gaussian-model}\label{sec:covariances}

Theorem~\ref{thm:final-information} connects final-model membership disclosure to per-iteration gradient disclosures $I(T_j;\widehat g_i\mid\theta_i)$. We now show how the exact per-iteration gradient disclosure within a fitted Gaussian background model increases with GNQ.

The adversary knows the target and population distribution $\mathcal P$, but not the realized non-target training datapoints. At an observed model $\theta_i=\theta$, it therefore needs a predictive distribution for their combined gradient contribution to distinguish target membership. This distribution describes the adversary's uncertainty about the unseen sample.

Jaynes' maximum-entropy estimate provides a criterion for constructing a distribution from the attacker's partial knowledge. It is a type of statistical inference that yields the least biased estimate possible on the information available to the attacker, remaining maximally noncommittal with regard to the information the attacker lacks~\citep{jaynes1957information}. When this knowledge is summarized by a mean and positive-definite covariance, the resulting estimate is Gaussian.

The background is a sum of contributions from unknown datapoints. For a fixed conditional law in which these contributions are independent and identically distributed with finite second moments, the multivariate CLT gives a Gaussian limit for their centered sum divided by the square root of the number of contributions. This statement holds at fixed gradient dimension and batch-selection probability as the number of background datapoints grows. It motivates a Gaussian estimate for the background. Applying this argument after observing $\theta_i=\theta$ requires conditional independence; observing a model trained on reused datapoints can introduce dependence between them. We therefore specify the finite-batch Gaussian law as a local predictive model.

LiRA illustrates how an attacker models observations under competing membership hypotheses: it fits Gaussians to logit-transformed true-label confidences from shadow models trained with and without the target, then evaluates a likelihood ratio~\citep{carlini2022lira}. Here, the observation is the mini-batch gradient. The auditor calibrates the background mean and covariance from the training gradients to evaluate disclosure under this predictive model.

\paragraph{Local model.}
Fix the observed model $\theta_i=\theta$. We study the per-iteration gradient disclosure $I(T_j;\widehat g_i\mid\theta_i=\theta)$, with membership probability $\rho_i=\Pr(T_j=1\mid\theta_i=\theta)$ and fresh target-batch inclusion $M_{ij}\sim\operatorname{Ber}(q)$, where $q=B/N_t$ and $0<\rho_i,q<1$. The mask is independent of $T_j$ at this fixed $\theta_i$. The target gradient $g_{ij}=\nabla_\theta\ell(\theta,d_j)$ is known, while the unseen non-target datapoints contribute the random background $W_i$:
\[
\widehat g_i=W_i+\frac1B T_jM_{ij}g_{ij},
\qquad
W_i=\frac1B\sum_{\substack{k=1\\k\ne j}}^{N_t}M_{ik}g_{ik}.
\]
The auditor calibrates the background covariance using the regularized second-moment sum
\[
S=\sum_{\substack{k=1\\k\ne j}}^{N_t}g_{ik}g_{ik}^{\top}+\lambda I_P,
\qquad GNQ_{ij}=g_{ij}^{\top}S^{-1}g_{ij},\qquad\lambda>0.
\]
We use $C_i=(q/B^2)S$ as the covariance estimate: the factor $q/B^2$ accounts for Bernoulli selection and gradient normalization, and $S$ uses uncentered second moments as a covariance approximation. (Under the population model, the background covariance also contains a mean term of relative order $q$, which $C_i$ omits.) With a fitted mean $m_i$, the local predictive model is
\begin{equation}\label{eq:predictive-background}
\begin{aligned}
W_i\mid\theta_i=\theta&\sim\mathcal N(m_i,C_i),\qquad C_i=\frac q{B^2}S,\\
W_i&\perp(T_j,M_{ij})\mid\theta_i=\theta.
\end{aligned}
\end{equation}
The fitted $S$ and $m_i$ are held fixed in this model; the realized training gradients remain private to the auditor. The common background law specifies the adversary's uncertainty under both membership hypotheses. Its conditional independence from the target term is a local modeling assumption, additional to the original independent draws of training datapoints.

Our goal is to show that $I(T_j;\widehat g_i\mid\theta_i=\theta)$ in this model increases with $GNQ_{ij}$ for fixed $\rho_i$ and $q$. Rescaling the observation by $C_i^{-1/2}$ makes the background a standard Gaussian and expresses the target's contribution in terms of GNQ.

\begin{theorem}[Per-iteration gradient disclosure and GNQ]\label{thm:gaussian-gnq}
Fix $\theta_i=\theta$, $0<\rho_i,q<1$, and a fitted $S\succ0$. In the predictive model~\eqref{eq:predictive-background}, the per-iteration gradient disclosure $I(T_j;\widehat g_i\mid\theta_i=\theta)$ depends on the target gradient through $GNQ_{ij}=g_{ij}^{\top}S^{-1}g_{ij}$ and increases strictly with this score. More precisely,
\begin{equation}\label{eq:exact-predictive-information}
I(T_j;\widehat g_i\mid\theta_i=\theta)
=I\!\left(T_j;\sqrt{GNQ_{ij}/q}\,T_jM_{ij}+Z\,\middle|\,\theta_i=\theta\right),
\end{equation}
where $Z\sim\mathcal N(0,1)$ is independent of $T_j$ and $M_{ij}$.
\end{theorem}

\begin{proof}
Subtracting $m_i$ and multiplying the observation by the invertible matrix $C_i^{-1/2}$ preserve mutual information. The background covariance becomes $C_i^{-1/2}C_iC_i^{-1/2}=I_P$, so the transformed observation is
\[
C_i^{-1/2}(\widehat g_i-m_i)
=\frac1B C_i^{-1/2}g_{ij}\,T_jM_{ij}+Z_P,
\qquad Z_P\sim\mathcal N(0,I_P).
\]
When $g_{ij}=0$, the observation is independent of membership and Equation~\eqref{eq:exact-predictive-information} is zero on both sides. Otherwise, rotate the first coordinate into the direction of $C_i^{-1/2}g_{ij}$. The remaining coordinates are independent of both membership and the first coordinate, so the per-iteration gradient disclosure is entirely determined by the first coordinate. Using $C_i^{-1}=(B^2/q)S^{-1}$ from Equation~\eqref{eq:predictive-background}, its squared shift is
\[
\frac1{B^2}g_{ij}^{\top}C_i^{-1}g_{ij}
=\frac1q g_{ij}^{\top}S^{-1}g_{ij}
=GNQ_{ij}/q,
\]
which proves Equation~\eqref{eq:exact-predictive-information}.
To prove that this per-iteration gradient disclosure increases with GNQ, let $Y_s=\sqrt{s}\,T_jM_{ij}+Z$ with $s=GNQ_{ij}/q$ at the fixed conditioning value. For $0\le s_1<s_2$, an observation with the law of $Y_{s_1}$ can be obtained from $Y_{s_2}$ by
\[
Y_{s_1}=\sqrt{s_1/s_2}\,Y_{s_2}+\sqrt{1-s_1/s_2}\,Z',
\]
where $Z'$ is an independent standard Gaussian. Data processing gives $I(T_j;Y_{s_1}\mid\theta_i=\theta)\le I(T_j;Y_{s_2}\mid\theta_i=\theta)$. Equality would require the membership posterior given $Y_{s_2}$ to remain recoverable after this transformation. Its Gaussian transition density is positive everywhere, so this requires a constant posterior. The two membership cases have different means for $s_2>0$ and $0<\rho_i,q<1$, making that posterior nonconstant. The inequality is therefore strict.
For known membership ($\rho_i=0$ or $1$), $I(T_j;\widehat g_i\mid\theta_i=\theta)=0$.
\end{proof}


\section{Algorithms and Computation}\label{app:algorithms}

\subsection{Time and Memory Complexity of GNQ Computation}
\label{app:bsghost_complexity}

In this section we analyze the efficiency of algorithms~\ref{alg:naive} and~\ref{alg:general_ghost}. The NaiveGNQ algorithm represents the direct implementation of Definition~\ref{def:gnq}.

\vspace{5 pt}
\Needspace{5\baselineskip}
\noindent
\textbf{Complexity: NaiveGNQ (reference implementation).}

\vspace{5 pt}
\noindent
\textbf{Time.}
Let $T_{\text{bwd}}$ be
the time cost of a standard backward pass. NaiveGNQ computes per-sample gradients $g_{ij}\in \mathbb{R}^P$ via ${N_t}$ separate backward passes:
\[
T_{\text{extract}}^{\text{Naive}} = O({N_t}\,T_{\text{bwd}}).
\]
To compute $GNQ_{ij} = g_{ij}^{\top} \big(S_{-j}^{(i)}\big)^{-1} g_{ij}$ with
$S_{-j}^{(i)} = \sum_{\substack{k=1\\k\ne j}}^{{N_t}} g_{ik}g_{ik}^{\top} +\lambda I_P$,
it (i) forms each $S_{-j}^{(i)}$ by summing ${N_t}-1$ outer products (cost $O(P^2)$ each), totaling $O({N_t}^2P^2)$,
and (ii) inverts ${N_t}$ distinct $P\times P$ matrices, costing $O({N_t}P^3)$.
\[
T_{\text{GNQ}}^{\text{Naive}} = O({N_t}^2P^2 + {N_t}P^3).
\]
Hence the total time is
\[T_{\text{Naive}} = O({N_t}T_{\text{bwd}} + {N_t}^2P^2 + {N_t}P^3).\]

\vspace{5 pt}
\noindent
\textbf{Memory.}
NaiveGNQ stores ${N_t}$ per-example gradient vectors (cost $O({N_t}P)$) and (at least) one dense $S\in\mathbb{R}^{P\times P}$ (cost $O(P^2)$), so
\[
M_{\text{Naive}} = O({N_t}P + P^2).
\]

\begin{figure}[t]
\centering

\begin{minipage}[t]{0.78\linewidth}
\vspace{0pt}

\begin{algorithm}[H]
\caption{NaiveGNQ (for reference)}
\small
\label{alg:naive}
\KwIn{Model at iteration $i$ ($\theta_i$), training datapoints $\{d_j\}_{j=1}^{{N_t}}$, $\lambda>0$}
\KwOut{$\mathrm{GNQ}_i \in \mathbb{R}_{\geq 0}^{{N_t}}$ : a vector contains all $GNQ_{ij}$ values}

\tcp{\textbf{Per-example gradients require ${N_t}$ backprops}}
\For{$j \gets 1$ \KwTo ${N_t}$}{
    Zero model gradients\;
    
    Backpropagate to compute and store $g_{ij} = \nabla_\theta \ell(\theta_i; d_j) \in \mathbb{R}^P$\;

}

\tcp{\textbf{GNQ requires $P\times P$ construction/inversion}}
$\mathrm{GNQ}_i \leftarrow 0_{{N_t}}$\;

\For{$j \gets 1$ \KwTo ${N_t}$}{
    Form $
        S_{-j}^{(i)} \gets \sum_{\substack{k=1\\k\ne j}}^{{N_t}} g_{ik} g_{ik}^{\top} + \lambda I_P \in \mathbb{R}^{P \times P}
    $\;
    
    Invert it $\big(S_{-j}^{(i)}\big)^{-1}$\;
    
    $GNQ_{ij} \gets g_{ij}^{\top} (S_{-j}^{(i)}\big)^{-1} g_{ij}$\;

    $\mathrm{GNQ}_i[j] \gets GNQ_{ij}$\;
}

\end{algorithm}

\end{minipage}
\end{figure}

\Needspace{3\baselineskip}
\vspace{5 pt}
\noindent
\textbf{Complexity: BS-Ghost}

\vspace{5 pt}
\noindent
BS-Ghost never forms $G$ and never constructs/inverts $P\times P$ matrices.
Ignoring bias terms, each single-position linear layer contributes
$K_\ell=(X_\ell X_\ell^{\top})\odot(\delta_\ell\delta_\ell^{\top})$,
where $X_\ell\in\mathbb{R}^{B\times n_{\text{in}}^{(\ell)}}$ and $\delta_\ell\in\mathbb{R}^{B\times n_{\text{out}}^{(\ell)}}$
are produced by the (already-required) training backward pass.
Thus GNQ adds no additional per-example backward passes beyond training.

\vspace{5 pt}
\noindent \textbf{Time.} For layer $\ell$, we use $n_{\mathrm{in}}^{(\ell)}$ and $n_{\mathrm{out}}^{(\ell)}$ to denote input/output widths.
Per layer, forming $X_\ell X_\ell^{\top}$ costs $O(B^2 n_{\text{in}}^{(\ell)})$ and
$\delta_\ell\delta_\ell^{\top}$ costs $O(B^2 n_{\text{out}}^{(\ell)})$; the Hadamard product is $O(B^2)$.
Summing across layers,
\[
T_{\text{kernel}}^{\text{Ghost}}
= O\!\left(B^2\sum_{\ell}\!\big(n_{\text{in}}^{(\ell)}+n_{\text{out}}^{(\ell)}\big)\right),
\qquad
T_{\text{solve}}^{\text{Ghost}} = O(B^3),
\]
The $B^3$ term is one matrix solve shared by all $B$ datapoints. The additional GNQ computation is
\[
T_{\text{Ghost}}
= O\!\left(B^2\sum_{\ell}\big(n_{\text{in}}^{(\ell)}+n_{\text{out}}^{(\ell)}\big)+B^3\right).
\]
\vspace{5 pt}
\noindent \textbf{Memory.}
BS-Ghost stores $K\in\mathbb{R}^{B\times B}$ (cost $O(B^2)$) and the cached layer tensors
(cost $O(B\sum_\ell(n_{\text{in}}^{(\ell)}+n_{\text{out}}^{(\ell)}))$), hence
\[
M_{\text{Ghost}}
= O\!\left(B^2 + B\sum_{\ell}\big(n_{\text{in}}^{(\ell)}+n_{\text{out}}^{(\ell)}\big)\right).
\]
\noindent
Crucially, there is no explicit dependence on ${N_t}$ and $P$ in these time and memory formulas.

\Needspace{6\baselineskip}
\subsection{\rev{plain}{Proof of the Batch-Space GNQ Identity}}
\label{app:batch-algebra}
Lemma~\ref{lem:loo_gnq} converts $h_j^{(i)}$ into GNQ by removing the target gradient from the matrix sum.

\begin{proof}
The matrices satisfy $S_{\mathrm{tot}}^{(i)}=S_{-j}^{(i)}+g_{ij}g_{ij}^{\top}$; $\lambda>0$ makes both invertible. Sherman--Morrison gives the inverse after adding the target's outer product:
\[
\big(S_{-j}^{(i)} + g_{ij}g_{ij}^{\top}\big)^{-1}
=
\big(S_{-j}^{(i)}\big)^{-1}
-\frac{\big(S_{-j}^{(i)}\big)^{-1} g_{ij} g_{ij}^{\top}\big(S_{-j}^{(i)}\big)^{-1}}
{1+ g_{ij}^{\top}\big(S_{-j}^{(i)}\big)^{-1} g_{ij}}.
\]
Write $x=g_{ij}^{\top}(S_{-j}^{(i)})^{-1}g_{ij}=GNQ_{ij}$. Multiplying the matrix identity on the left by $g_{ij}^{\top}$ and on the right by $g_{ij}$ turns its left side into $h_j^{(i)}$. On the right, the first term is $x$ and the fraction has numerator $x^2$ and denominator $1+x$. Hence
\[
h_j^{(i)}=x-\frac{x^2}{1+x}=\frac{x(1+x)-x^2}{1+x}=\frac{x}{1+x}.
\]
Since $x\ge0$, this also gives $0\le h_j^{(i)}<1$. Rearranging $h_j^{(i)}(1+x)=x$ yields $h_j^{(i)}=x(1-h_j^{(i)})$. Dividing by the positive $1-h_j^{(i)}$ proves $GNQ_{ij}=x=h_j^{(i)}/(1-h_j^{(i)})$.
\end{proof}

\subsection{Ghost Kernels for Additional Layer Types}
\label{app:ghost-kernels}
The linear-layer kernel is given in Section~\ref{sec:ghostgnq}. The cached errors must match the gradient of each datapoint's loss. If the batch loss is averaged over $B$ datapoints, the errors are scaled by $1/B$ and the kernel by $1/B^2$. Multiplying the errors by $B$ or using $\lambda/B^2$ with the scaled kernel gives the same GNQ.

\paragraph{Convolutional and sequence layers.}
For an unfolded convolution or a linear weight matrix shared across positions, let $x_{js}$ and $\delta_{js}$ be example $j$'s activation and output error at position $s$. Its weight gradient sums the contributions from all positions: $\sum_s\delta_{js}x_{js}^{\top}$. Taking the inner product of two such sums pairs every position $s$ in datapoint $j$ with every position $t$ in datapoint $k$. Each pair factors into an activation inner product and an error inner product, giving
\[
K_{jk}^{W}=\sum_{s,t}\langle\delta_{js},\delta_{kt}\rangle\langle x_{js},x_{kt}\rangle.
\]
A shared bias adds $\langle\sum_s\delta_{js},\sum_t\delta_{kt}\rangle$. Computing only same-position terms omits cross-position contributions. The same rule applies to tied parameters used in multiple layers: first sum contributions to the shared parameter, then compute its Gram entry once.
\paragraph{Embedding layers.}
The gradient of an embedding lookup is nonzero only in its token's row. For one token $t_j$ per datapoint, two embedding gradients overlap only when their tokens match:
\[
K_{jk}^{(\ell)}=\langle\delta_j^{(\ell)},\delta_k^{(\ell)}\rangle\,\mathbb I\{t_j=t_k\}.
\]
Here $\mathbb I$ is one for a token match and zero otherwise. In a sequence, repeated occurrences add to the same embedding row, so the datapoint-level kernel sums every matching pair of positions:
\[
K_{jk}^{(\ell)}=\sum_{s,t}\langle\delta_{js}^{(\ell)},\delta_{kt}^{(\ell)}\rangle\,\mathbb I\{t_{js}=t_{kt}\}.
\]
A sequence-level score therefore sums all token gradients before forming the kernel.

\section{Experimental Details and Additional Results}
\label{app:original-details}\label{app:additional-experiments}\label{app:evidence}

\subsection{\rev{plain}{Efficiency of BS-Ghost: Experimental Setup}}\label{app:efficiency-setup}
Table~\ref{tab:gpt2-ghost-overhead} uses GPT-2 Small ($P\approx1.24\times10^8$) on WikiText-2~\citep{merity2016pointer}, with batch size $B=16$ and sequence length 128. Table~\ref{tab:mlp-naive-vs-ghost} uses a $100$--$40$--$10$ MLP ($P=4{,}450$) on MNIST~\citep{lecun1998mnist}, with $B=32$. Both experiments run on a single NVIDIA Tesla P100-SXM2 GPU with 16GB memory. We measure GPT-2 iteration time and throughput, and MLP time and peak memory.

\begin{table}[!t]\centering\small
\begin{tabular}{lrrr}\toprule
Model&Parameters&Training peak&Extra GNQ peak\\\midrule
GPT-2 Small&124M&5641 MiB&$<0.1$ MiB\\
GPT-2 Medium&355M&12594 MiB&$<0.1$ MiB\\\bottomrule
\end{tabular}
\caption{Peak GPU memory during training. GNQ adds less than 0.1 MiB to the peak.}
\label{tab:added-peak-memory}
\end{table}

\subsection{Disclosure Beyond Membership: Image Reconstruction}
\label{app:class-reconstruction}
The class score sums GNQ over training datapoints in that class and training iterations. Reconstruction distance in Figure~\ref{fig:class-reconstruction}(a,b) is the minimum squared Euclidean distance in penultimate-layer activations between the attack output and any same-class training image.
\begin{figure}[!t]\centering
\begin{subfigure}{.47\linewidth}\includegraphics[width=\linewidth]{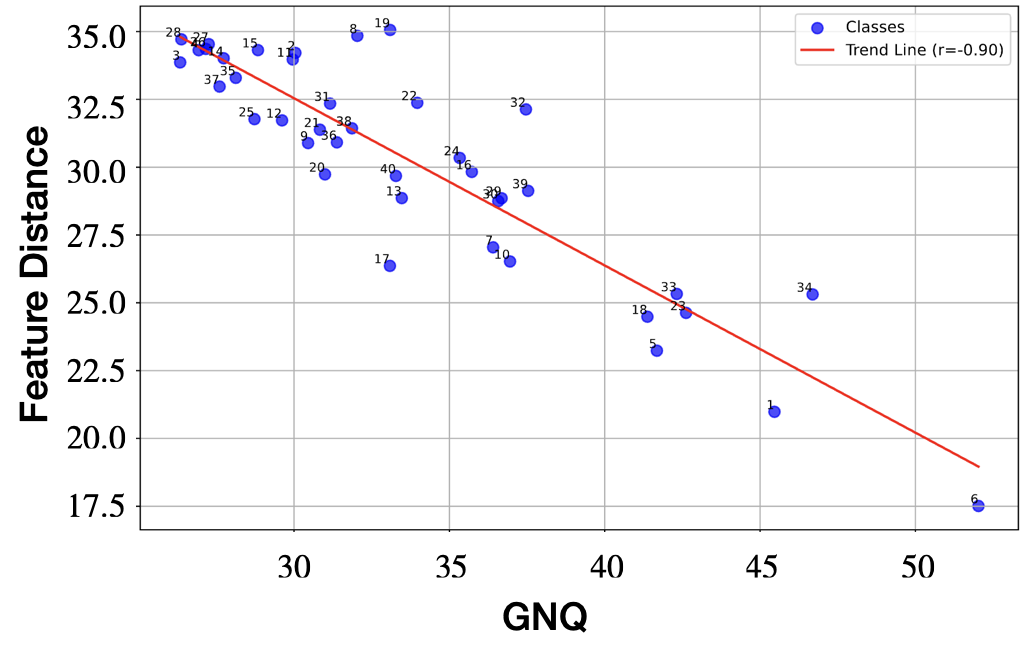}\caption{Faces / MLP.}\end{subfigure}\hfill
\begin{subfigure}{.47\linewidth}\includegraphics[width=\linewidth]{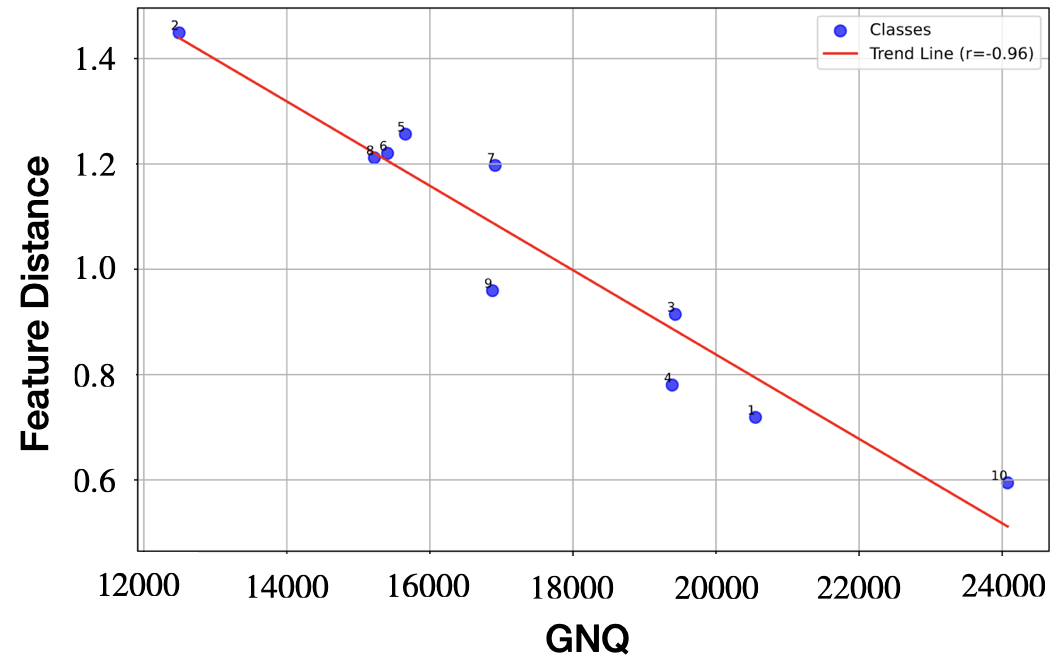}\caption{MNIST / ResNet.}\end{subfigure}\par
\begin{subfigure}{.47\linewidth}\includegraphics[width=\linewidth]{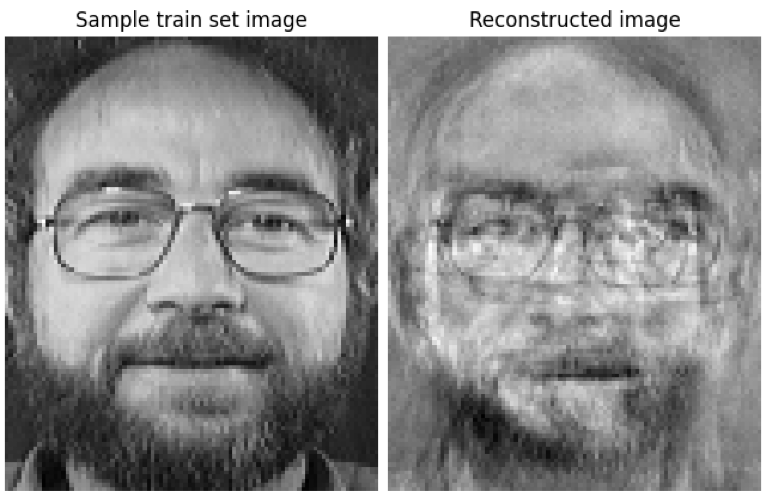}\caption{Lowest-GNQ Faces class (28).}\end{subfigure}\hfill
\begin{subfigure}{.47\linewidth}\includegraphics[width=\linewidth]{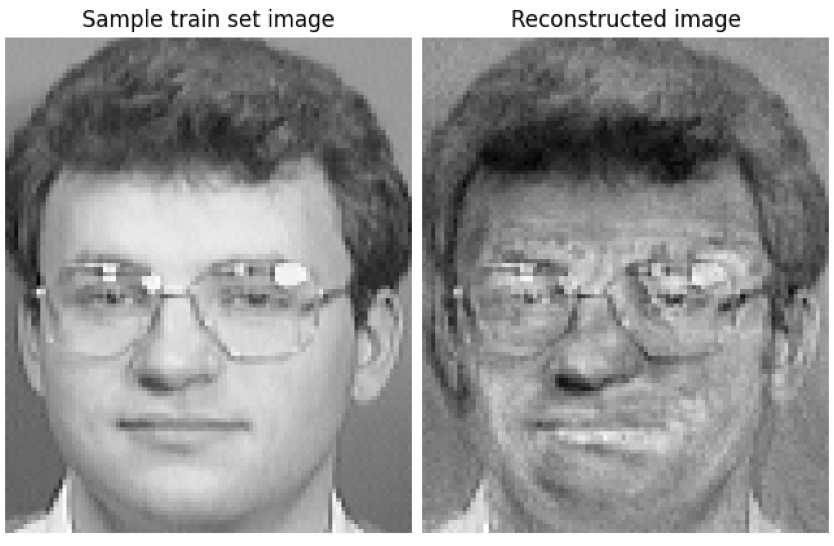}\caption{Highest-GNQ Faces class, one epoch.}\end{subfigure}
\caption{High-GNQ classes are easier to reconstruct. Faces class 6 (highest GNQ) is recognizable after one epoch at 10--30\% accuracy and after 30 epochs at about 95\% (Figure~\ref{fig:reconstruction-examples}(b)).}
\label{fig:class-reconstruction}
\end{figure}

\subsection{Memorization: Text Extraction and Common Knowledge}\label{app:text-extraction}\label{app:sentence-examples}
\rev{plain}{With a commercial LLM, we construct} a list of 200 sentences that represent common knowledge.  Examples include:
``William Shakespeare was a famous English playwright and poet.'' ``Water freezes at 0 degrees Celsius (or 32 Fahrenheit) under standard atmospheric pressure.'' `` Alexander Graham Bell is credited with inventing the first practical telephone.''

We then construct 200 surprising false assertions that contradict common knowledge. Examples are:
``The American polka-dotted squirrel sustains itself entirely by photosynthesizing moonlight.''
``Mount Everest is an alien spacecraft that crashed into Earth during the late Jurassic period.''
``Authentic Italian spaghetti grows on kelp-like vines in the deep, underwater caves of the Mediterranean Sea.''

\begin{table}[!t]\centering\small
\begin{tabular}{lrrr}\toprule
Selection & Inspected & Surprising/false & Common facts\\\midrule
GNQ &20&20&0\\
GNQ &40&32&3\\
GNQ &80&58&6\\\midrule
Counterfactual &20&7&0\\
Counterfactual &40&19&0\\
Counterfactual &80&39&2\\\midrule
Random (trial 1)&20&4&1\\
Random (trial 2)&20&7&2\\\bottomrule
\end{tabular}
\caption{\rev{plain}{Recovered common facts and surprising false assertions by selection method ($n=400$).}}
\label{tab:ghostgnq-targeted-extraction}
\vspace{-1em}
\end{table}

\subsection{How Other Training Choices Affect Disclosure}\label{app:training-choices}
The CIFAR-10/CNN sweeps use Adam and the Salem/ML-Leaks MIA \citep{salem2018mlleaks}, evaluated by AUC-ROC. We take each datapoint's maximum epoch-end GNQ and average these maxima to obtain the plotted mean GNQ. The model-size panel additionally normalizes this quantity. We hold $\lambda$ fixed across datapoints and sweeps to isolate the effects of the training choices.

\begin{figure}[!t]\centering
\begin{subfigure}{.47\linewidth}\includegraphics[width=\linewidth]{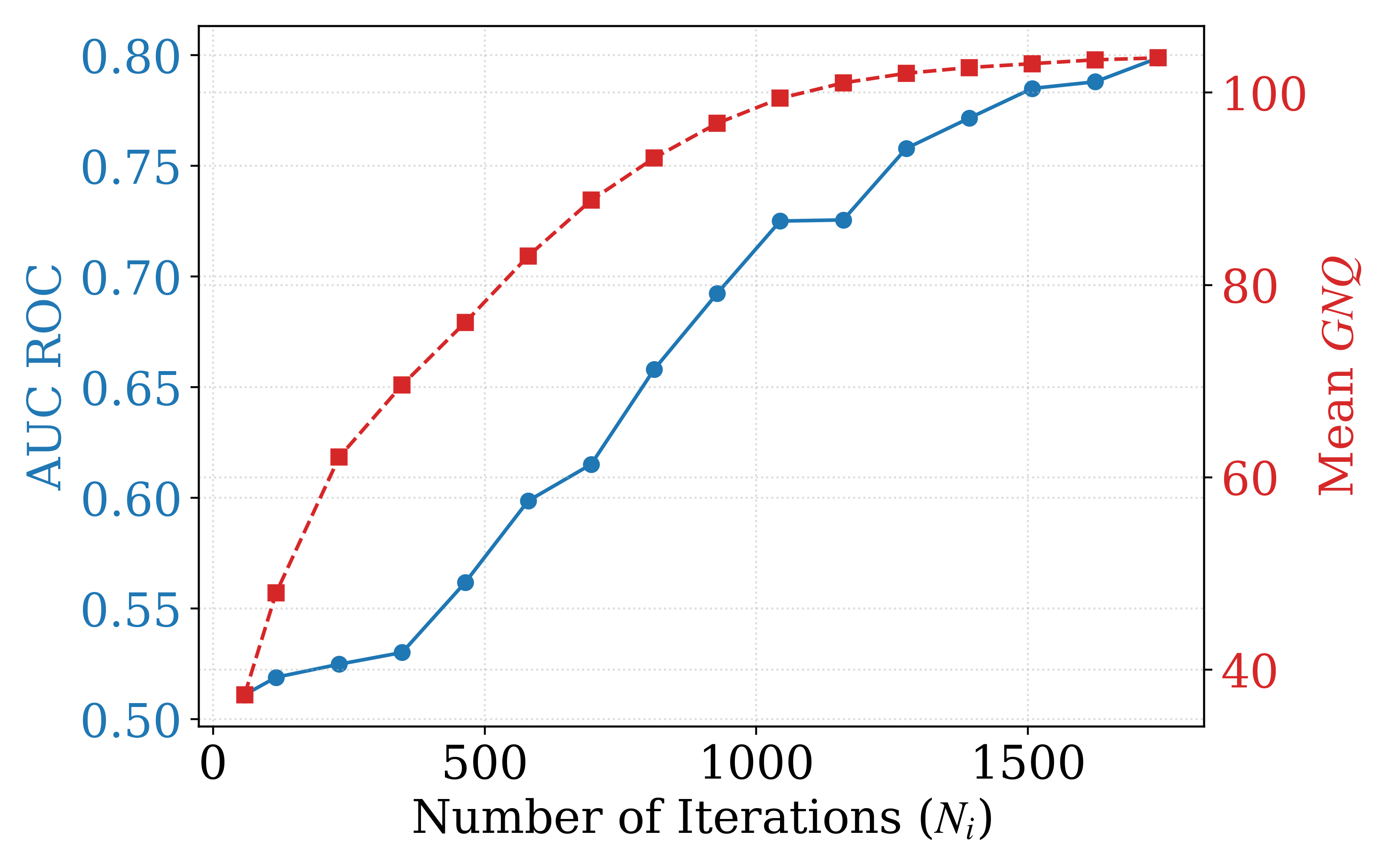}\caption{Number of training iterations.}\end{subfigure}\hfill
\begin{subfigure}{.47\linewidth}\includegraphics[width=\linewidth]{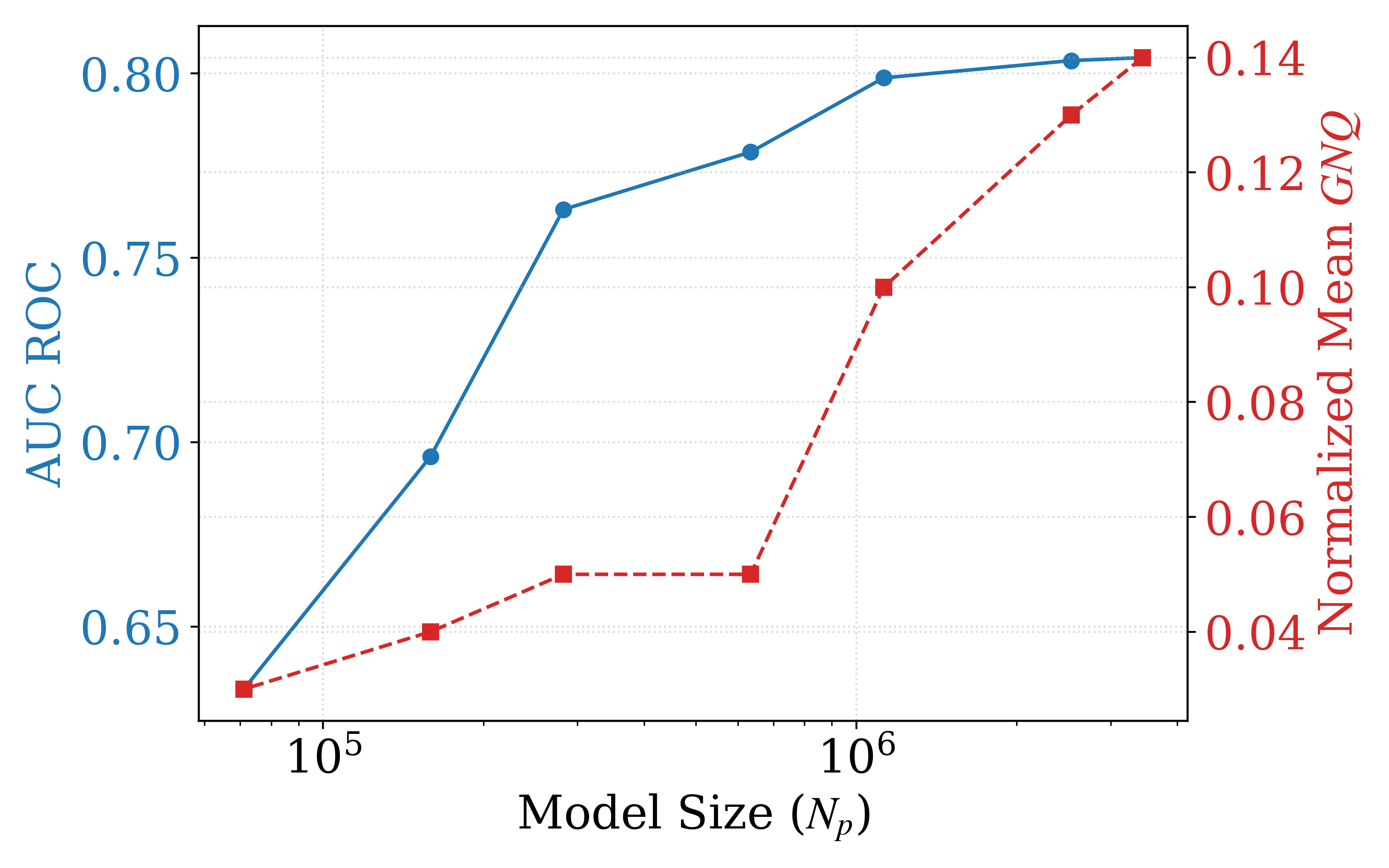}\caption{Model size.}\end{subfigure}\par\vspace{2pt}
\begin{subfigure}{.47\linewidth}\includegraphics[width=\linewidth]{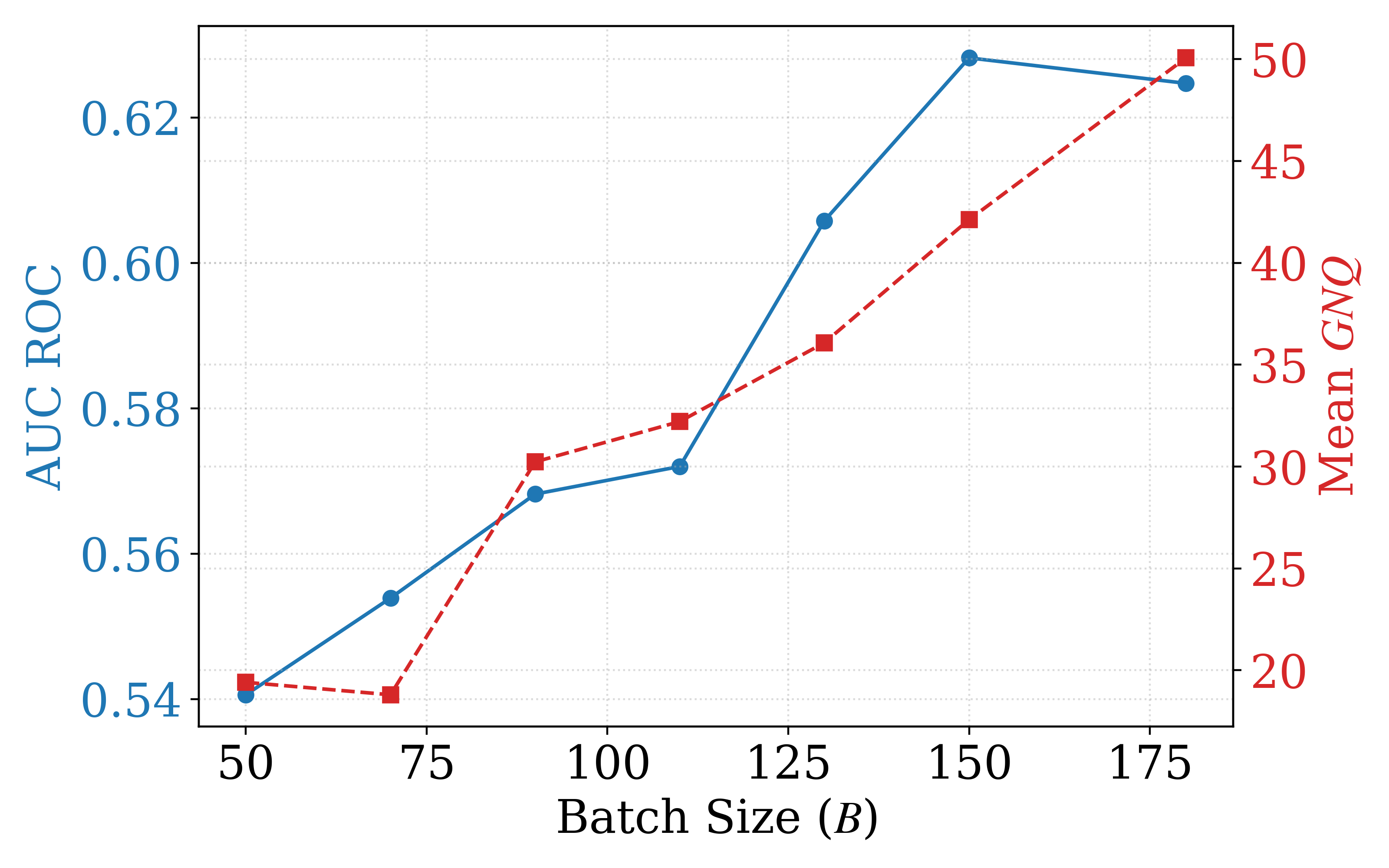}\caption{Batch size.}\end{subfigure}\hfill
\begin{subfigure}{.47\linewidth}\includegraphics[width=\linewidth]{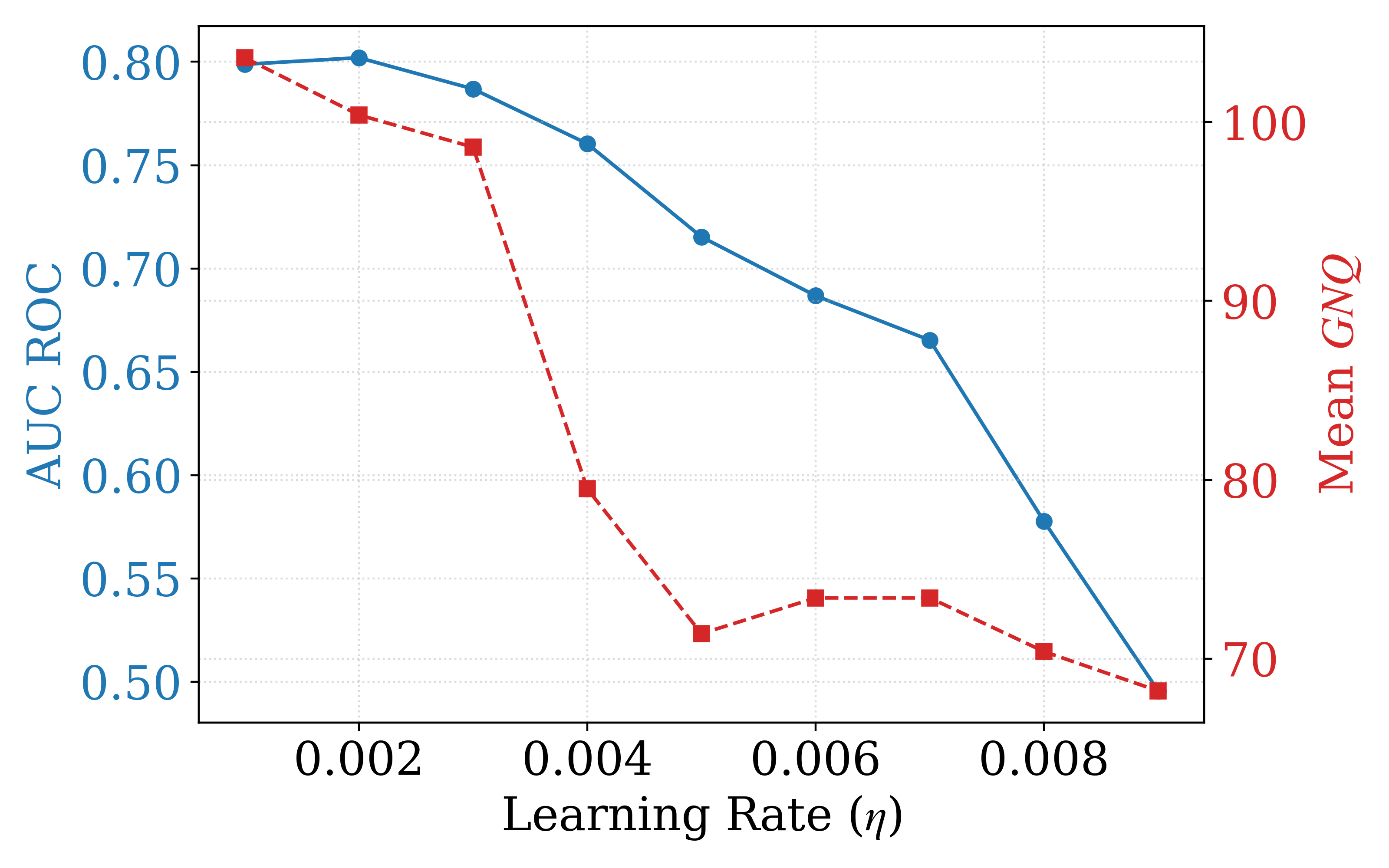}\caption{Learning rate.}\end{subfigure}\par
\caption{GNQ and MIA performance across four training choices. Blue: AUC-ROC, left axis. Red: mean GNQ, right axis, normalized for model size. Both rise overall with iterations, model size, and batch size, and fall with learning rate.}
\label{fig:training-sweeps}
\end{figure}
We vary training from 50 to 1700 iterations, model size from 71 thousand to 3.4 million parameters, and learning rate from 0.001 to 0.009 in steps of 0.001.

\subsection{Disclosure During LLM Training: Experimental Setup}\label{app:trajectory-setup}
We use a publicly available GPT-2 model with many intermediate checkpoints~\citep{gonzalez_gpt2_ag_news_hf}. This GPT-2 model was trained for 100 epochs on AG-News. We then randomly select \rev{plain}{a set} of training datapoints ($n = 500$).
For each checkpoint (epochs $\{1,10,20,\dots,100\}$), we compute per-example GNQ using a fixed mini-batch partition (fixed order) so GNQ is comparable across checkpoints.
Table~\ref{tab:gnq_traj_texts} summarizes GNQ trajectories by averaging over the 11 checkpoints.

\begin{table*}[!t]
\centering
\small
\begin{tabular}{c c p{0.70\linewidth}}
\toprule
\textbf{Trajectory} & \textbf{Avg.\ GNQ} & \textbf{Training text} \\
\midrule
High \#1 & 12{,}973 &
Sleek Looks and Superb Performance to Woo Mac Fans and PC Buyers Most all-in-one desktops with LCD panels (the Sony VAIO and Gateway Profile systems come to mind) are two units permanently connected together: the part of the case housing the motherboard and drives, and the monitor... \\
High \#2 & 5{,}334 &
EU seeks revamp budget pact The European Commission on Friday suggested ways of revamping the EU's Stability and Growth Pact and said the proposed changes would restore the credibility of its much-flouted budget discipline rules... \\
\midrule
Low \#1  & 543 &
UN warns of Sudan refugee exodus Some 30,000 Sudanese refugees might cross into Chad to escape persecution by Arab militia, the UN warns... \\
Low \#2 & 551 &
Philadelphia Considers Wireless Internet for All By DAVID B. CARUSO PHILADELPHIA (AP) -- For about \$10 million, city officials believe they can turn all 135 square miles of Philadelphia into the world's largest wireless Internet hot spot... \\
\bottomrule
\end{tabular}
\caption{\rev{plain}{Training texts for the two highest- and two lowest-average GNQ trajectories in Fig.~\ref{fig:gnq_ckpt_trajectories}.}}
\label{tab:gnq_traj_texts}
\end{table*}

\Needspace{5\baselineskip}
\subsection{Connecting Disclosure to Training Data Attribution}\label{app:influence-comparison}\label{app:qwen-removal}

Influence functions estimate how reweighting training datapoints changes a chosen output or loss~\citep{koh2017influence,grosse2023influence}. GNQ measures a datapoint's gradient contribution against interference from other gradients, connecting it to membership disclosure (Equation~\eqref{eq:gnq-bound}). Projecting gradients onto a behavior vector focuses the Qwen score on the selected trait, testing an initial connection to training data attribution (Section~\ref{sec:experiments}).

BS-Ghost reuses the training forward and backward passes and solves in batch space. EK-FAC approximates curvature but still needs candidate gradients and attribution for each query.

\noindent\textbf{Qwen2.5-7B: computational efficiency.}
In Figure~\ref{fig:qwen-runtime}, the EK-FAC and Concept Influence attribution times exclude inverse-Hessian preparation~\citep{kowal2026concept}. GNQ's timing measures extra training work; baseline attribution times cover 1000 Misaligned Opinions datapoints.

GNQ-guided removal and retraining of Qwen2.5-7B reduce Bad Medical Advice (Figure~\ref{fig:qwen-medical}) and Misaligned Opinions (Figure~\ref{fig:qwen-removal}) more than Vector Filter or Projection Difference at every tested nonzero removal fraction. Both traits are scored on a 0--1 scale.

\begin{figure}[!t]
\centering\includegraphics[width=.52\linewidth]{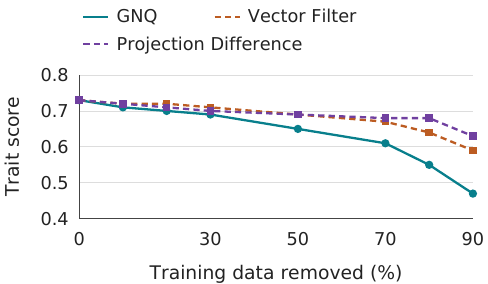}
\caption{Misaligned Opinions after GNQ-guided removal and retraining of Qwen2.5-7B. Lower scores are better. Bad Medical Advice is shown in Figure~\ref{fig:qwen-medical}.}
\label{fig:qwen-removal}
\end{figure}

\end{document}